\documentclass{article}
\usepackage[utf8]{inputenc}
\usepackage[margin=1in]{geometry}

\usepackage{macros}
\usepackage{url}

\title{Simple, Fast, and Flexible Framework for Matrix Completion with Infinite Width Neural Networks}
\author{Adityanarayanan Radhakrishnan\thanks{Laboratory for Information \& Decision Systems, and 
 Institute for Data, Systems, and Society, 
 Massachusetts Institute of Technology} $~~$ George Stefanakis\textsuperscript{\specificthanks{1}} $~~$ Mikhail Belkin \thanks{
Halıcıoğlu Data Science Institute, University of California, San Diego} $~~$ Caroline Uhler\textsuperscript{\specificthanks{1},}\thanks{Broad Institute of MIT and Harvard} }
\date{\today}

\begin{document}

\maketitle

\begin{abstract}
 Matrix completion problems arise in many applications including recommendation systems, computer vision, and genomics. Increasingly larger neural networks have been successful in many of these applications, but at considerable computational costs.  Remarkably, taking the width of a neural network to infinity allows for improved computational performance.  In this work, we develop an infinite width neural network framework for matrix completion that is simple, fast, and flexible.   Simplicity and speed come from the connection between the infinite width limit of neural networks and kernels known as neural tangent kernels (NTK).   In particular, we derive the NTK for fully connected and convolutional neural networks for matrix completion.  The flexibility stems from a \emph{\prior}, which allows encoding relationships between coordinates of the target matrix, akin to semi-supervised learning.  The effectiveness of our framework is demonstrated through competitive results for virtual drug screening and image inpainting/reconstruction.  We also provide an implementation in Python to make our framework accessible on standard hardware to a broad audience.
\end{abstract}

\section{Introduction}
Matrix completion is a fundamental problem in machine learning, arising in a variety of applications from collaborative filtering to virtual drug screening, and  image inpainting/reconstruction.  Given a matrix $Y$ with only a subset of coordinates observed, the goal of matrix completion is to impute the unobserved entries in $Y$.  For example, in collaborative filtering (Fig.~\ref{fig: Overview Schematic}a), matrix completion is used to infer the interests of a user from the interests of other users.  A prominent example is the Netflix challenge of inferring movie preferences from sparsely-populated matrices of user ratings~\cite{NetflixPrize}.
%A prominent example is the Netflix challenge of inferring movie ratings 
For virtual drug screening (Fig.~\ref{fig: Overview Schematic}b), matrix completion is used to predict the effect of a drug on a cell type/state given other drug and cell type/state combinations.  For image inpainting (Fig.~\ref{fig: Overview Schematic}c) and image reconstruction (Fig.~\ref{fig: Overview Schematic}d), matrix completion is used to restore missing pixels in a corrupted image.

Standard approaches to matrix completion such as nuclear norm minimization \cite{MinNuclearNormMinRank, MatrixCompletionNuclearNorm, MatrixCompletionNuclearNormTao} or deep matrix factorization \cite{DeepMatrixFactorization} aim for a completion that yields a low rank matrix.  While such methods can be effective in applications like collaborative filtering,  where low rank can capture user similarity, such an objective function can lead to ineffective solutions for applications including drug response imputation, image inpainting, or image reconstruction.  For example, in the case of drug response imputation, imputing a new drug would involve predicting the values of an entirely-missing vector of gene responses (in contrast to the aforementioned Netflix problem, which involves imputing single scalar entries of the matrix). 
%In this sense, the drug imputation problem is more complex than the aforementioned Netflix problem, as the imputation of interest involves predicting entire vectors, as opposed to single scalar entries in a matrix. 
%For example, in the case of drug response imputation, any new drug would have the entire corresponding column missing.  Thus, 
In this case, a low-rank reconstruction would replace all missing entries with a fixed constant, thereby leading to poor predictive performance.  Similarly, for image inpainting and reconstruction, a low rank completion is generally ineffective since it does not take into account local image structure \cite{RegularizedDMFImages, DepthImageInpainting}.  Thus, there is a need for a more general approach to matrix completion that can easily adapt to the structures in different applications.  

In this work, we provide a simple, fast, and flexible framework for matrix completion. To accomplish this, 
%In order to provide a general and flexible framework, 
we view matrix completion as an inverse problem; given a matrix $Y \in \mathbb{R}^{m \times n}$ such that a subset of coordinates $S = \{(i, j)\} \subset [m] \times [n]$ are observed and the other entries are missing, we aim to construct $\hat{Y} \in \mathbb{R}^{m \times n}$ such that $\hat{Y}_{i,j} \approx Y_{i,j}$ for all observed coordinates $(i,j)\in S$.  %In this work, we provide a simple, fast, and flexible framework for matrix completion by using 
We use neural networks to model the observations in $Y$ and use gradient descent to minimize:
\begin{equation}
\label{eq: Deep Matrix Factorization Framework}
      \mathcal{L}(\mathbf{W})\!=\!\!\!\sum_{(i,j) \in S} \!\!\!\!\left(Y_{i,j} \!-\! [W_d \phi(W_{d-1}  \phi( \ldots W_2  \phi(W_1  Z) \ldots ))]_{i,j}  \right)^2\!\!\!,
\end{equation}
where $\mathbf{W}=\{W_{\ell}\}_{\ell=1}^{d}$ are the weights of a neural network with each $W_\ell \in \mathbb{R}^{k_{\ell+1} \times k_\ell}$ and $k_{d+1} =  m$, $k_1 = p$; $\phi: \mathbb{R} \to \mathbb{R}$ is a fixed element-wise nonlinearity; and $Z \in \mathbb{R}^{p \times n}$ is a fixed application-dependent matrix, which we call the feature prior (described in detail below).  The completed matrix $\hat{Y}$ is then obtained using the forward model with the trained weights, i.e.,  $\hat{Y}= W_d \phi(W_{d-1}  ( \ldots W_2  \phi(W_1  Z) \ldots ))$. The main contribution of this work is showing that minimizing the loss in Eq.~[\ref{eq: Deep Matrix Factorization Framework}] when the width $\{k_{\ell}\}_{\ell=2}^{d}$ of the neural network tends to infinity, gives rise to a simple, fast, and flexible framework for matrix completion suitable for a range of applications.

Superficially, the formulation in Eq.~[\ref{eq: Deep Matrix Factorization Framework}] appears similar to that of traditional supervised learning, where a neural network is trained to map data (which would correspond to $Z$ in our formulation) to corresponding labels $Y$. However, it is important to note that in our formulation $Z$ can be independent of the observations $Y$ ($Z$ could for example be the identity matrix or a random matrix). Thus, $Z$ should be interpreted as a prior that can be chosen in an application-dependent manner. We will discuss the effect of this prior as well as how to choose it for very different applications like virtual drug screening and image inpainting.  

% Our framework subsumes prior neural network approaches for particular problems including low rank matrix factorization \cite{DeepMatrixFactorization} and image inpainting/reconstruction \cite{DeepImagePrior}, where, \textcolor{red}{as we will show; ok to add this? I don't think they will otherwise get e.g. why identity corresponds to low rank}, the \prior is the identity matrix %,  where the \prior is and a random matrix, respectively. 

\begin{figure*}[!t]
    \centering
    \includegraphics[width=1\textwidth]{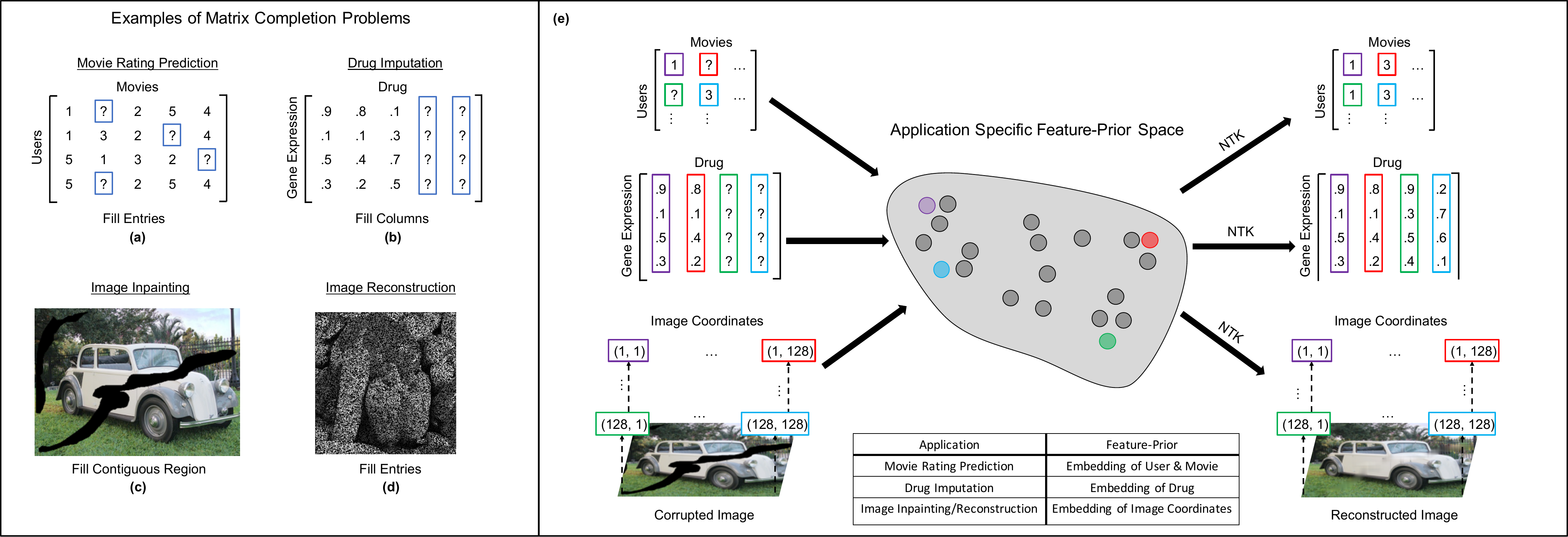}
    \caption{An overview of matrix completion applications where ?'s in (a), (b) and zero (black) pixels in (c), (d) represent unobserved entries.  (a) Collaborative filtering example (the Netflix problem), where the goal is to predict how a user would rate (on a scale of 1-5) an unseen movie.  (b) Virtual drug screening, where the problem is to predict the gene expression profile for an unobserved drug / cell type combination. In this application entire columns are unobserved. % We do not have access to any gene expression values for a drug on a new cell type and hence entire columns are unobserved.  
    (c,d) Image inpainting and reconstruction involves reconstructing a corrupted region of an image (shown as black pixels). (e) Our NTK matrix completion framework is easily adapted to solve all of the above problems by selecting a \prior that represents an embedding of application specific metadata.}
    \label{fig: Overview Schematic}
\end{figure*}

\subsubsection*{Simple and Fast Algorithm for Matrix Completion through Infinite Width Networks}

A trend for improving neural network performance is to make models larger (in multiple respects) \cite{KaplanScalingLaws, ResNet, UNet, WideResNets}. Underscoring this trend, several recent works have empirically demonstrated the advantage of larger (in particular, wider) networks with respect to generalization and performance for classification and representation learning tasks~\cite{RethinkingGeneralization, belkin2019reconciling, DeepDoubleDescent, Adit_AEs}.  There is also an emerging theoretical understanding of the benefit of larger models \cite{TwoModelsDoubleDescent, SurprisesHighDimensional, BenignOverfitting}. The extreme case where network width approaches infinity, is what we consider in this paper in the setting of matrix completion.

While generally larger neural networks require more computational resources for training, quite unintuitively, the limit as network width approaches infinity may yield computational savings.  Namely, it was recently shown that training infinite width networks is equivalent to solving kernel regression with a particular kernel known as the neural tangent kernel (NTK)~\cite{NTKJacot}. For fully connected networks, the NTK can be computed efficiently in closed form \cite{NTKJacot}, and thus training an infinite width network reduces to solving a linear system.  While this may still be computationally expensive when the number of examples is large, we will use recent pre-conditioner methods~\cite{EigenPro, EigenProGPU, Falkon} to overcome this limitation. 

For convolutional networks no efficient computation of the NTK (the so-called CNTK) has been known~\cite{CNTKArora, BayesianDeepImagePrior, DenoisingNTK}. A major contribution of this work is to provide a memory and runtime efficient algorithm for computing the exact CNTK for matrix completion for a class of practical neural network architectures. As a consequence, our framework can be used to inpaint or reconstruct high-resolution images with hundreds of thousands of pixels.  We also provide software for constructing the CNTK as well as pre-computed kernels. The simplicity and speed of our framework is exhibited by the fact that most of the results in this work require only a CPU and can be run efficiently on a laptop.

\subsubsection*{Flexibility through Feature Prior}

The matrix $Z$ in Eq.~[\ref{eq: Deep Matrix Factorization Framework}] is key to making our framework easily adaptable to different applications.  Unlike traditional supervised learning where the goal is to learn a mapping from data $X$ to labels $Y$, the matrix $Z$ in our framework can be independent of the observations in $Y$.  We refer to $Z$ as a \textbf{\prior} since, as we will see, by minimizing the loss in Eq.~[\ref{eq: Deep Matrix Factorization Framework}], the entries of $Z$ encode structure between the coordinates of $Y$ (see Fig.~\ref{fig: Overview Schematic}e).  

We will demonstrate the flexibility of our framework by using it in two very different applications, namely for drug response imputation and image inpainting/reconstruction.  For drug response imputation, we will select \priors that encode information about cell and drug type combinations.  For image inpainting and reconstruction, we will select \priors %with independent, identically distributed (i.i.d) entries and U-Net architectures from \cite{DeepImagePrior, BayesianDeepImagePrior}, which, as we will prove, 
that encode information about image coordinates.  In addition to being flexible, we will show that our approach is competitive in terms of speed and accuracy with prior approaches that were specifically developed for drug response imputation \cite{SontagDrugImputation, FaLRTC} or image \mbox{inpainting/reconstruction~\cite{BiharmonicInpainting,ScikitImage,DeepImagePrior}.}

\section{Matrix Completion with the NTK} 
\label{section: NTK for Matrix Completion}
% \textcolor{red}{you need to make clear what is known and what isn't about the NTK. It needs to be clear throughout this section when you talk about known results and these are for the classification setting, and what has not been known and why matrix completion is a different problem and requires new results.} \textcolor{blue}{Adit: I hope I've addressed this a bit better now, but let me know if there is anything I can do to make it a bit clearer (especially the last two subsections).}

% \misha{let's discuss this, still confusing to me}

In this section, we derive the NTK for matrix completion when using fully connected networks. Our derivation provides a principled method for selecting the \prior, $Z$; %, for matrix completion with fully connected networks.  
namely, we will show that $Z$ should be an embedding of \textit{coordinate metadata}, i.e. information describing the coordinates of $Y$. For example in drug response imputation, each column of $Z$ could correspond to a different drug and two columns of $Z$ should be similar if the drug metadata is similar (e.g. the molecular structures are similar). The resulting method is then equivalent to performing semi-supervised learning to map from the columns of $Z$ to observed entries in each row of $Y$.  In Section~\ref{sec: Drug Imputation with the NTK}, we will utilize this theoretical result to select an effective \prior for virtual drug screening.  

Since the NTK forms the backbone of our framework, we start with the definition of the NTK~\cite{NTKJacot} and briefly review how solving kernel regression with the NTK connects to training infinitely wide neural networks.  

\begin{definition}[NTK] Let $f(w ; x): \mathbb{R}^{p} \times \mathbb{R}^{d} \to \mathbb{R}$ denote a neural network with parameters $w$.  The corresponding \textbf{neural tangent kernel}, $K:\mathbb{R}^{d} \times \mathbb{R}^{d} \to \mathbb{R}$, is a symmetric, continuous, positive definite function given by: 
\begin{align*}
    K(x, x') = \langle \nabla_{w} f(w^{(0)} ; x), \nabla_{w} f(w^{(0)} ; x') \rangle,
\end{align*}
where $w^{(0)} \in \mathbb{R}^{p}$ are the network parameters at initialization.  
\end{definition}

For a review of kernel regression and kernel functions see \cite{KernelMachineReview}.  Given training data $(x^{(i)}, y^{(i)})\in\mathbb{R}^d\times \mathbb{R}$ for $i=1,\dots , n$,  %$((x^{(1)}, \dots , x^{(n)}), y) \subset \mathbb{R}^{d \times n} \times \mathbb{R}^{1 \times n}$, 
solving kernel regression with the NTK involves minimizing the loss: 
\begin{align}
\label{eq: Kernel Regression}
    \mathcal{L}(\alpha) = \| y - \alpha \hat{K} \|_2^2, 
\end{align}
where $\alpha \in \mathbb{R}^{1 \times n}$, $y=[y^{(1)}, \dots , y^{(n)}]^T$, and $\hat{K} \in \mathbb{R}^{n \times n}$ with $\hat{K}_{i,j} = K(x^{(i)}, x^{(j)})$. The work of \cite{NTKJacot} established that using kernel regression with the NTK is equivalent %to map $x^{(i)}$ to $y^{(i)}$  
(under mild assumptions) to training a neural network to map $x^{(i)}$ to $y^{(i)}$ using the mean squared error, in the limit as the network width tends to infinity.  Throughout this work, we will assume that $w_i^{(0)} \overset{i.i.d}{\sim} \mathcal{N}(0,1)$ and that the nonlinearity $\phi$ in Eq.~[\ref{eq: Deep Matrix Factorization Framework}] is homogeneous (which includes, for example, the rectified linear unit (ReLU), a widely used nonlinearity) so that the NTK corresponding to a fully connected network can be computed efficiently in closed form \cite{CosineKernel, NTKJacot, LeakyReLUNTK}; see Appendix \ref{appendix: A} for a short review of the relevant literature %to introduce the relevant 
and notation.

\subsection*{Feature Prior Provides a Flexible Approach for Matrix Completion through Connection with Semi-supervised Learning}
% \label{sec: NTK Derivation for Matrix Completion}

A natural approach for imputing missing entries in a matrix, $Y$, is to first obtain an embedding of the coordinates of $Y$ (e.g. a map from coordinates $(i,j)$ to $\mathbb{R}^{p}$) and then learn a map from the coordinate embedding to the observed entries in $Y$ (e.g. a map from $\mathbb{R}^{p}$ to $Y_{i,j} \in \mathbb{R}$); see also~\cite[Ch.1]{RecommenderSystemsBook}. For example, for virtual drug screening, one could first embed the drugs based on their molecular properties and then learn a map from this embedding to the measured output, such as gene expression. %for movie rating prediction, one can learn a mapping from an embedding of movies and users to observed ratings \textcolor{blue}{(Adit: Should I give a different example like drugs here?)}.  
Such an approach in which a map is learned from an embedding to the observed samples is referred to as \textit{semi-supervised learning} \cite[Ch.15]{goodfellow2016deep}.  In this section, we will prove that minimizing the loss in Eq.~[\ref{eq: Deep Matrix Factorization Framework}] is equivalent to using a semi-supervised learning approach for matrix completion.  Namely, we show that the columns of $Z$ represent an embedding of the coordinates of $Y$ and that the NTK is used to map from the columns of $Z$ to the entries in $Y$. 

It is a priori unclear how to compute the NTK for matrix completion, since this requires training examples and labels.  For this, we note the following equivalent formulation of Eq.~[\ref{eq: Deep Matrix Factorization Framework}]:
\begin{align}
\label{eq: Matrix Completion Trace Formulation}
    \mathcal{L}(\mathbf{W}) &= \sum_{(i, j) \in S} (Y_{i, j}  - \langle f_Z(\mathbf{W}), M_{\{(i,j)\}} \rangle )^2, \\
    f_Z(\mathbf{W}) &= W^{(d)} C_d \phi(W^{(d-1)} C_{d-1}\phi( \ldots W^{(2)}C_2\phi(W^{(1)} Z) ) \ldots ), \nonumber
\end{align}
where $C_\ell = c/\sqrt{k_\ell}$ for a constant $c$, $\langle A, B \rangle = tr(A^TB)$ denotes the trace inner product, and  $M_{\{(i,j)\}}$ is an indicator matrix, i.e., it has a $1$ in the $(i,j)$ entry and zeros everywhere else.  To ease notation, we will use $M_{ij}$ to denote the indicator matrix $M_{\{(i,j)\}}$. The formulation in Eq.~[\ref{eq: Matrix Completion Trace Formulation}] shows that we can view matrix completion as a problem where the "training examples" are indicator matrices $M_{ij}$ and the "labels" are the corresponding entries $Y_{i,j}$. This reformulation yields the following closed form % The following theorem presents a closed form 
for the NTK for matrix completion, where %  In the theorem below, 
$\check{\phi}: [-1, 1] \to \mathbb{R}$ denotes the dual activation function \cite{DualActivation} to $\phi$. To keep notation simple, we here provide the theorem when $\phi$ is the ReLU activation function, but this result holds generally for homogeneous nonlinearities; see Appendix \ref{appendix: B}.  %\textcolor{purple}{Adit: is the motivation for ReLU below now a bit better?}

\begin{theorem}
\label{theorem: d layer relu ntk} 
Assume $Z = \{z^{(i)}\}_{i=1}^{n} \in \mathbb{R}^{p \times n}$, where each column is normalized with $\|z^{(i)} \|_2 = 1$.  Let $f_Z(\mathbf{W})$ be a $d$ layer fully connected network with nonlinearity $\phi(x) = \max (x, 0)$ and $c= \sqrt{2}$ in Eq.~[\ref{eq: Matrix Completion Trace Formulation}].  Then, as widths $k_2, k_3, \ldots, k_{d} \to \infty$, the NTK for matrix completion with $f_Z(\mathbf{W})$ is given by
\begin{align*}
    K(M_{ij}, M_{i'j'}) =  \begin{cases}
    \kappa_d({z^{(j)}}^T z^{(j')}) ~~~ \textrm{if } i = i' \\
    0 ~~~~~~~~~~~~~~~~~~ \textrm{if } i \neq i' 
    \end{cases},
\end{align*}
where $\kappa_d(\xi) = \check{\phi}^{(d)}(\xi) + \kappa_{d-1}(\xi) \frac{d\check{\phi}}{d\xi} (\check{\phi}^{(d-1)}(\xi))$, and $\check{\phi}^{(h)}(\xi) = \check{\phi}( \check{\phi}^{(h-1)}(\xi))$ for $h \geq 1$ and $\check{\phi}^{(0)}(\xi) = \xi$.
\end{theorem}

The proof as well as an example showing %how to compute and using 
how Theorem~\ref{theorem: d layer relu ntk} can be used in practice to compute the NTK for matrix completion is presented in Appendix \ref{appendix: B}.  Since the kernel value between $M_{ij}$ and $M_{i'j'}$ is a function of columns $j$ and $j'$ of $Z$, Theorem~\ref{theorem: d layer relu ntk} implies that the NTK for matrix completion maps columns of $Z$ to entries $Y_{i,j}$, and thus the columns of $Z$ encode structure between the coordinates of $Y$.  

By varying the nonlinearity $\phi$, depth $d$, and \prior~$Z$, our framework encapsulates a variety of semi-supervised learning approaches. To provide a non-trivial example, we prove in Appendix \ref{appendix: B} that our framework for matrix completion generalizes Laplacian-based semi-supervised learning \cite{SemiSupervisedLearningManifolds}.  This insight regarding the connection between our framework for matrix completion and semi-supervised learning represents the backbone for a simple and competitive approach to virtual drug screening described in the next section.  

\begin{figure*}[!t]
    \centering
    \includegraphics[width=\textwidth]{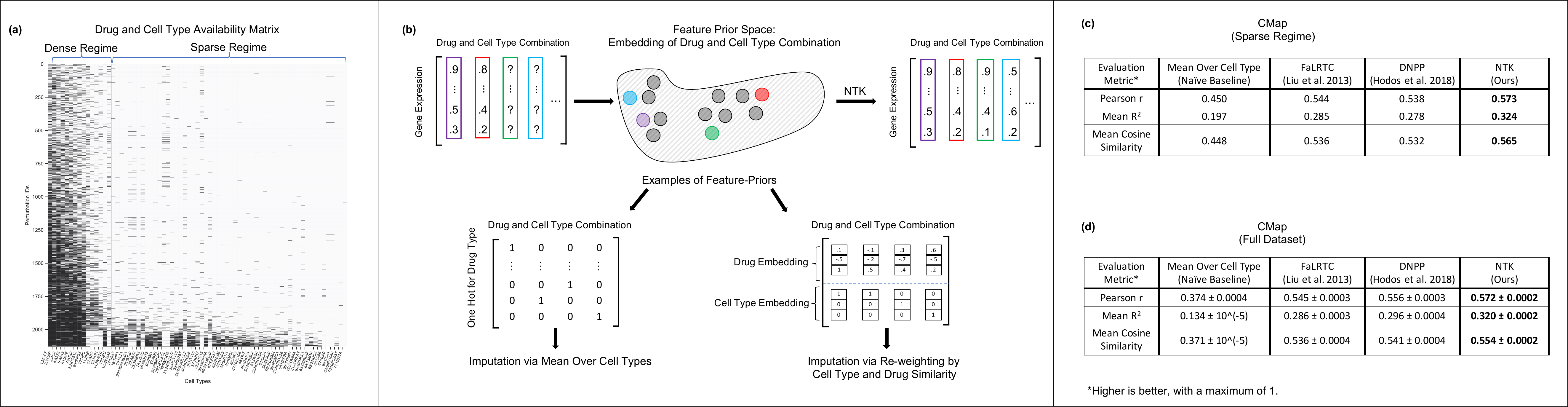} 
    \caption{Our infinite width neural network framework outperforms DNPP \cite{SontagDrugImputation}, FaLRTC \cite{FaLRTC}, and mean over cell types for drug response imputation on CMap.  (a) We visualize the availability of cell type and drug combinations of the subset from \cite{SontagDrugImputation}.  (b) Our method corresponds to first providing an embedding of cell type and drug combinations as the \prior and then applying the NTK.  We show that: (1) using a \prior consisting of one-hot vectors for drugs corresponds to imputation by performing mean across observations for each cell type and (2) using a \prior that captures similarity between drugs and cell types is effective for imputation. (c, d) Our infinite width neural network framework (denoted NTK) outperforms DNPP and mean over cell type across three evaluation metrics.  We use 5 rounds of 10-fold cross validation to determine that the difference between our method and the next best method, DNPP, is statistically significant (p-value less than $10^{-20}$).}
    \label{fig: NTK Drug Imputation}
\end{figure*}

\section{Virtual Drug Screening with the NTK}
\label{sec: Drug Imputation with the NTK}

% Predicting the gene expression of drug and cell type combinations has important applications in drug re-purposing \cite{COVIDAutoencoding} and drug discovery \textcolor{red}{TODO}. While CMap contains the gene expression profiles for over a million samples, the space of drug and cell type combinations remains largely unexplored \cite{SontagDrugImputation}. Thus, imputation of the gene expression values for new drug and cell type combinations remains a fundamental problem in genomics. 

% \textcolor{red}{You need to introduce CMap and the problem of drug imputation first.}
%We now use our theoretical results from Section \ref{section: NTK for Matrix Completion} to impute the effect of a drug.
%the expression of landmark genes for drug and cell type combinations based on the CMap drug screen \cite{CMap}.  20,413 drugs, 72 cell types
CMAP is a prominent, large-scale, publicly available drug screen that considers 20,413 different compounds and 72 different cell lines \cite{CMap}. Experiments in CMAP were performed on a subset of 201,484 drug/cell line pairs; for each of these pairs the gene expression profile of $978$ landmark genes was measured.
CMAP has been an important resource for computational approaches to drug discovery and drug repurposing \cite{CMap,COVIDAutoencoding, DrugRepurposingReview}.  In these applications, the goal is to use a subset of observed drug/cell type pairs to predict the gene expression profile of new drug/cell type pairs.  These profiles are then used to identify drug candidates of interest that can be tested experimentally \cite{DrugRepurposingLungCancer, DrugRepurposingAlzheimers}. 

The CMAP dataset can be viewed as a 3-dimensional tensor (drugs, cell lines, genes), where many of the entries are missing. In the following, we will use the same pre-processing of the data as in \cite{SontagDrugImputation} to filter out drug/cell line combinations with very few or inconsistent samples; a description and a link to the dataset is provided in SI Appendix C.   The resulting drug/cell line combinations are shown in Fig.~\ref{fig: NTK Drug Imputation}a. The 3-dimensional tensor can be flattened into a matrix, where the columns correspond to drug/cell line combinations and the rows represent genes (see Fig.~\ref{fig: NTK Drug Imputation}b)\textcolor{black}{; i.e., following the notation from Section 2, entry $Y_{ij}$ of the resulting flattened matrix is a real-valued number quantifying the gene expression of gene $i$ in drug and cell type combination $j$.} 
%In particular, following the notation from Section \ref{section: NTK for Matrix Completion}, entry $Y_{ij}$ of the resulting flattened matrix represents the gene expression for gene $i$ and for drug and cell type combination $j$.}
This matrix has a missing column for every missing drug/cell line combination. Classical low rank matrix factorization methods would prove ineffective in this setting since they would replace each missing column by the same constant column. % to achieve a low rank reconstruction.  
On the other hand, Theorem \ref{theorem: d layer relu ntk} suggests the NTK as an effective way for imputing the missing gene expression profiles by selecting the feature prior $Z$ such that two columns of $Z$ are similar if they correspond to similar drug/cell line pairs. In the following, we discuss three different feature priors for this application; for a full description of these priors see SI Appendix D.

\subsection*{Feature Prior corresponding to the Mean Over Cell Type Baseline}%Imputation with a Mean Over Cell Type Baseline Prior} 
A simple baseline is to impute the gene expression profiles for each missing drug for a given cell line by the mean over all observed drugs for this cell line. % \cite{SontagDrugImputation, SyntheticInterventionsCMap}.  
Quite surprisingly, this simple approach gives rise to a strong baseline~\cite{SontagDrugImputation, SyntheticInterventionsCMap}, since cell type is the dominant factor, while drugs have subtle effects on gene expression.  %As we prove in Appendix~E, this baseline corresponds to using a \prior that embeds drugs via one-hot vectors. 

% \textcolor{red}{there is too much emphasis in this section on reference 26. Please look for relevant references where CMAP has been analyzed that have appeared in journals that are competitive to PNAS (Nature Comm, PLoS Comp Bio, etc., and refer to these. In particular in this subsection it's easy to refer to other papers other than 26)}  

While it is generally nontrivial to improve upon this simple baseline without constructing a specialized algorithm \cite{SontagDrugImputation, DrugScreeningDeepLearningFramework, DrugScreeningTensorDecompTT, CarpenterDrugImputation}, our NTK framework provides an easy way for doing so.  In particular, our framework makes it evident that the \prior corresponding to the mean over cell type baseline is trivial, since it corresponds to an embedding in which drugs are encoded via one-hot vectors %orthogonal 
(see Appendix \ref{appendix: E}).  Thus, to improve upon this baseline, we select any \prior that can capture similarities between drugs. 

\subsection*{\PriorAllCaps Corresponding to Previous Algorithms}

We now demonstrate that our framework provides a direct approach to improve on previous methods for virtual drug screening by using the output of previous methods as a \prior in our framework.  Namely, if a method is used to produce an imputation, $\hat{Y}$, then the columns in $\hat{Y}$ should represent an embedding of drug and cell type combinations that captures their similarity. Hence, we can use  $Z = \hat{Y}$ as the \prior in our method. For illustration, we apply this approach to two state-of-the-art methods for virtual drug screening: (1) Drug Neighbor Profile Prediction (DNPP)~\cite{SontagDrugImputation}, which is a weighted nearest neighbor scheme, and (2) Fast Low Rank Tensor Completion (FaLRTC) \cite{FaLRTC}, which involves low rank matrix completion along each slice of the CMAP tensor. We show that our framework using these feature priors yields an improvement over the individual methods; see Appendix \ref{appendix: F}. %\textcolor{red}{I think it's important to do this for 2 methods}. 

%show that using the imputation provided by DNPP as the \prior yields an improvement for virtual drug screening. 

%The work of \cite{SontagDrugImputation} utilized two methods for virtual drug screening: (1) Drug Neighbor Profile Prediction (DNPP), which is a a weighted nearest neighbor scheme and (2) Fast Low Rank Tensor Completion (FaLRTC).  

% \textcolor{red}{Let's chat about this. In figure c and d above, can you change the order: first mean over cell type, then FaLRTC, then DNPP and then NTK? Also, since you added  FaLRTC, we should talk about this in the text. I think I would make a subsection called Other Feature Priors Corresponding to Previous Algorithms and here discuss the low-rank feature prior as well as how algorithms like DNPP can be used to construct a feature prior. Then I would do 1 section called: Proposed Feature Prior for Imputation with the NTK and combine the two subsections into 1. Do you understand why DNPP doesn't work well in the sparse regime? Is the reasoning similar to why low-rank doesn't work well there? I'd give some intuition.}

\subsection*{Proposed \PriorAllCaps for Drug Response Imputation}

Observing the pattern of data availability in Fig.~\ref{fig: NTK Drug Imputation}a, it is apparent that a subset of cell lines have observations for many ($>150$) drugs (dense regime), while many cell lines have observations for only few ($\leq 150$) drugs (sparse regime).  While previous methods such as DNPP are quite effective in the dense regime, they are not as effective in the sparse regime; see Fig.~\ref{fig: NTK Drug Imputation}c and Appendix \ref{appendix: G}.  %For those drug/cell type pairs with very few examples, 
This can be explained by the fact that in the sparse regime DNPP roughly imputes using the simple mean over cell type baseline.

For effective drug response imputation in the sparse regime, our framework can be used to construct a simple \prior by concatenating embeddings for cell types and drugs.  In particular, we can use the gene expression values for a reference cell type for which there are a lot of drug observations (e.g.~MCF7 in CMAP) as the embedding of drugs and the mean gene expression across all observations for a given cell type as the embedding of cell type.  Fig.~\ref{fig: NTK Drug Imputation}c shows that the NTK with this simple \prior outperforms mean over cell type, FaLRTC and DNPP %both the mean over cell type, the fast low rank tensor completion (FaLRTC) algorithm \cite{FaLRTC}, and the Drug Neighbor Profile Prediction (DNPP) method from \cite{SontagDrugImputation} 
in the sparse regime.  We compare across Pearson r value, mean $R^2$, and mean cosine similarity.  A description of all evaluation metrics is provided in Appendix \ref{appendix: H}.  By combining our \prior for the sparse regime with the FaLRTC based \prior for the dense regime, we obtain a drug imputation method that significantly outperforms DNPP, FaLRTC, and mean over cell type on the full dataset; see Fig.~\ref{fig: NTK Drug Imputation}d (p-value less than $10^{-20}$ based on 5 rounds of 10-fold cross validation, with an improvement on every fold of every round across all metrics; see Appendix~I).

\section{Matrix Completion with the Convolutional NTK}
\label{sec: Image Inpainting CNTK}

While we have thus far derived and applied the NTK for matrix completion using fully connected networks, these architectures are not nearly as effective as convolutional networks for matrix completion tasks in which the target matrix is an image.  Similar to the case of fully connected networks, a closed form for the NTK corresponding to convolutional networks (the so-called CNTK) is known in the regression setting \cite{CNTKArora}, but it has not been considered in the setting of matrix completion.  Moreover, %since 
the runtime for computing the CNTK for regression scales quadratically with each image dimension. %, computing the CNTK for high resolution images is computationally expensive.  
In this section, we derive the CNTK for matrix completion and provide a computationally efficient method for computing the CNTK for matrix completion for a class of \priors that are effective for image inpainting and reconstruction.

% \textcolor{red}{You need to make clear again what is known and what isn't. e.g. Similar to the setting with fully connected networks, a closed form for the NTK corresponding to convolutional networks (the so-called CNTK) is known in the regression setting, but no efficient implementation is yet available. In this section, we derive the CNTK for matrix completion and provide a computationally efficient method...} In this section, we derive and analyze the NTK for convolutional networks (known as the CNTK \textcolor{REF? Is it CNTKArora? I removed this reference from below, so it should go somewhere.}) used for matrix completion.  Our analysis again illustrates the flexibility of our framework by connecting matrix completion with the CNTK to semi-supervised learning.  Additionally, we use our analysis to derive a computationally efficient method for computing the CNTK for matrix completion for a class of \priors that are effective for image inpainting and reconstruction.  

We begin by deriving the CNTK for matrix completion for a simple class of convolutional networks, when there are no downsampling or upsampling layers. We show that in this setting, the CNTK for matrix completion can be computed using terms from the CNTK for classification.  In the following proposition (proof in Appendix \ref{appendix: J}), $\Theta^{(d)} \in \mathbb{R}^{m \times n \times m \times n}$ denotes the tensor corresponding to the CNTK of a $d$ layer convolutional network in the classification setting \cite[Sec.~4]{CNTKArora}.
\begin{prop}
Let $f_Z(\mathbf{W})$ be a $d$ layer convolutional network used to map from \prior, $Z \in \mathbb{R}^{c \times m \times n}$, to the target matrix, $Y \in \mathbb{R}^{m \times n}$. Then as the number of convolutional filters per layer approaches infinity, the CNTK of $f_Z(\mathbf{W})$ is given by:
\begin{align}
\label{eq: CNTK for image completion}
    K(M_{ij}, M_{i'j'}) = [\Theta^{(d)}(Z, Z)]_{i,j,i',j'},
\end{align}
where $M_{ij}, M_{i'j'} \in \mathbb{R}^{m \times n}$ denote indicator matrices. 
\end{prop}

\subsection*{CNTK Performs Semi-Supervised Learning using Image Coordinate Features}

In Section \ref{section: NTK for Matrix Completion}, we established a connection between semi-supervised learning and matrix completion using the NTK.  We now establish a similar connection between semi-supervised learning and matrix completion with the CNTK for a class of \priors defined in Theorem \ref{theorem: Coordinate only CNTK} below. This class includes feature priors that are heavily used in image inpainting applications, namely where the channels of $Z$ are drawn i.i.d.~from a stationary distribution~\cite{BayesianDeepImagePrior, DeepImagePrior}. The following theorem (proof in Appendix \ref{appendix: K}), which is analogous to Theorem \ref{theorem: d layer relu ntk} for the NTK, implies that using the CNTK for matrix completion is equivalent to mapping from coordinate features to observed entries in the target matrix $Y$.

\begin{theorem}
\label{theorem: Coordinate only CNTK}
Consider a convolutional network of depth $d$ with homogeneous activation and in which all filters have size $q$ and circular padding.  Let $Z \in \mathbb{R}^{c \times m \times n}$ satisfy:
\begin{align*}
\sum_{\ell=1}^{c} \sum_{ -\alpha \leq a, b \leq \alpha } Z_{\ell, i+a, j+b} Z_{\ell, i' + a, j' + b} = \psi ( | i - i'|, | j - j'|) 
\end{align*}
for some $\psi: \mathbb{R}^2 \to \mathbb{R}$ with maximum at $(0,0)$ and $\alpha = \frac{q-1}{2}$ (odd $q$). Then as the number of convolutional filters per layer goes to infinity, the CNTK simplifies to:
\begin{align*}
    K(M_{ij}, M_{i'j'}) = \tilde{\psi}(|i-i'|, |j-j'|),
\end{align*}
where $\tilde{\psi}: \mathbb{R}^{2} \to \mathbb{R}$ is a function that can be computed from $\psi$ (a recursive formula is provided in Appendix \ref{appendix: K}). 
\end{theorem}

% The condition on $Z$ is satisfied for the uniform random inputs used heavily in \cite{DeepImagePrior, BayesianDeepImagePrior} as $c \to \infty$.\footnote{We can use the law of large numbers to evaluate the sum in Theorem \ref{theorem: Coordinate only CNTK} provided that $Z$ is normalized by $\frac{1}{\sqrt{c}}$.}  

Since the function $\tilde{\psi}$ depends only on the positions of the coordinates, Theorem \ref{theorem: Coordinate only CNTK} shows that the CNTK for matrix completion is equivalent to semi-supervised learning using kernels on features corresponding to coordinates.

\subsection*{Closed Form for the CNTK of Modern Architectures for Matrix Completion} Unlike the convolutional networks considered thus far, state-of-the-art architectures for unsupervised image inpainting such as~\cite{DeepImagePrior, BayesianDeepImagePrior} incorporate a variety of layer structures including strided convolution, nearest neighbor and bilinear upsampling, skip connections, and batch normalization.  %In order to mimic these neural networks as closely as possible, 
We derive (in Appendix \ref{appendix: L}) the CNTK for matrix completion using convolutional networks with the following layer structures: (1) Downsampling through Strided Convolution ; (2) Nearest Neighbor Upsampling ; and (3) Bilinear Upsampling.\footnote{The impact of linear downsampling and upsampling on the CNTK is briefly described in Appendix F of \cite{DenoisingNTK}, but the explicit forms are not computed nor used in the experiments.} 

% \textcolor{red}{You say above that you will now consider (1) and (2), but in the next sentence you say that you will do something different. Do we need the following sentence?} For the case of networks with nearest neighbor downsampling and upsampling layers, we next provide a runtime and memory efficient algorithm to compute and use the CNTK for matrix completion.  

\subsection*{Efficient Computation of the CNTK of Modern Architectures for Matrix Completion}

A key insight that we use to speed up the computation of the CNTK is that the kernel in Eq.~[\ref{eq: CNTK for image completion}] depends only on the \prior and not on the values of the observed pixels in an image.  Hence, the CNTK need  only be computed once for all images of a given resolution.  This enables a drastic speedup over recomputing the kernel for every new image, as is currently required in classification.  

% \textcolor{red}{Do you mean: But such a direct approach for computing the CNTK exactly is still prohibitive in terms of computation and memory for high resolution images. }
However, using such a direct approach to compute the CNTK is still computationally prohibitive for high resolution images.  In particular, computing the CNTK for a network with $d$ convolutional layers to complete an image of size $2^p \times 2^q$, requires $O(p^2q^2d)$ runtime and $O(2^{2p + 2q})$ space.  In order to overcome these limitations, prior work \cite{DenoisingNTK} used the Nyström method \cite{Nystrom} to approximate the kernel.  Instead of relying on such approximations, we here present an algorithm for computing the exact CNTK in a memory and runtime efficient manner for any convolutional neural network with circular padding, strided convolution, and nearest neighbor upsampling layers, when using a \prior with i.i.d.~random entries.  Such networks and \priors are heavily used for image completion tasks \cite{DeepImagePrior}.

% \textcolor{red}{Could your approach be extended to the regression setting? Otherwise, I think you may want to say something like: Instead of relying on such approximations, we make use of a critical difference between the regression and matrix completion setting present an algorithm for computing the exact CNTK in a memory and runtime efficient manner. This algorithm extends to convolutional networks with strided convolution and nearest neighbor upsampling layers.}

Our main insight that enables such an algorithm is that for convolutional networks with strided convolution and nearest neighbor upsampling layers, the CNTK for low resolution images can be expanded to high resolution images for any \prior with i.i.d.~random entries. In particular, if a neural network with $s$ downsampling and upsampling layers is used to inpaint images of resolution $2^{p} \times 2^{q}$, our algorithm requires only an array of size $2^{2s + p + q}$ while storing the full CNTK requires an array of size $2^{2p + 2q}$.  In practice, $s$ is exponentially smaller than $p, q$ and so our method is significantly more memory efficient; see the following specific example. In addition, since our method only requires computing the CNTK for images of size $2^{s+1} \times 2^{s+1}$, the runtime of our method is $O(2^{4s})$ instead of $O(2^{2p + 2q})$,  and thus, our method is significantly faster than a direct computation. %Since computing the CNTK for low resolution images requires less runtime and memory than computing the CNTK for high resolution images, our algorithm enables fast image completion.  
A detailed description and proof of our expansion algorithm is presented in Appendix \ref{appendix: M}. 

% when the convolutional architecture can be applied to both images of resolution $d_1$ and $d_2$  with $d_2 > d_1$, we can expand the kernel for resolution $d_1$,  $K_{d_1} \in \mathbb{R}^{d_1 \times d_1 \times d_1 \times d_1}$, to a tensor of size $\mathbb{R}^{d_1 \times d_1 \times d_2 \times d_2}$, which can be indexed to match the entries of the kernel for resolution $d_2$, $K_{d_2} \in  \mathbb{R}^{d_2 \times d_2 \times d_2 \times d_2}$. 
\textit{Example.} Let $f_Z(\mathbf{W})$ represent a convolutional neural network with circular padding, 3 layers of strided convolution with a stride size of $2$ in each direction, and 3 nearest neighbor upsampling layers with a \prior $Z \in \mathbb{R}^{c \times 512 \times 512}$ satisfying:
\begin{align*}
    \sum_{p=1}^{c} Z_{p, i, j} Z_{p, i', j'} = \begin{cases}
    C_1 & i = i' ~,~ j = j' \\
    C_2 & \text{otherwise}
\end{cases},
\end{align*}
where $C_1, C_2>0$ are constants.  Suppose $f_Z(\mathbf{W})$ is used to inpaint images of size $512 \times 512$.  Then, by computing the CNTK for $16 \times 16$ resolution images, $K_{\ell} \in \mathbb{R}^{16^2 \times 16^2}$, we can expand up to the exact CNTK for $512 \times 512$ images.  Computing $K_{\ell}$ takes roughly $11$ seconds when using a CPU with 1 thread and $\tilde{K}$ uses less than $100$MB of memory with floating point precision.  On the other hand, even storing the true kernel $K \in \mathbb{R}^{512^2 \times 512^2}$ would require roughly 256GB memory when using floating point precision.  This is twice the amount of RAM available on our server and 16 times the amount of RAM available on most laptops.      

\begin{figure*}[!t]
    \centering
    \includegraphics[width=\textwidth]{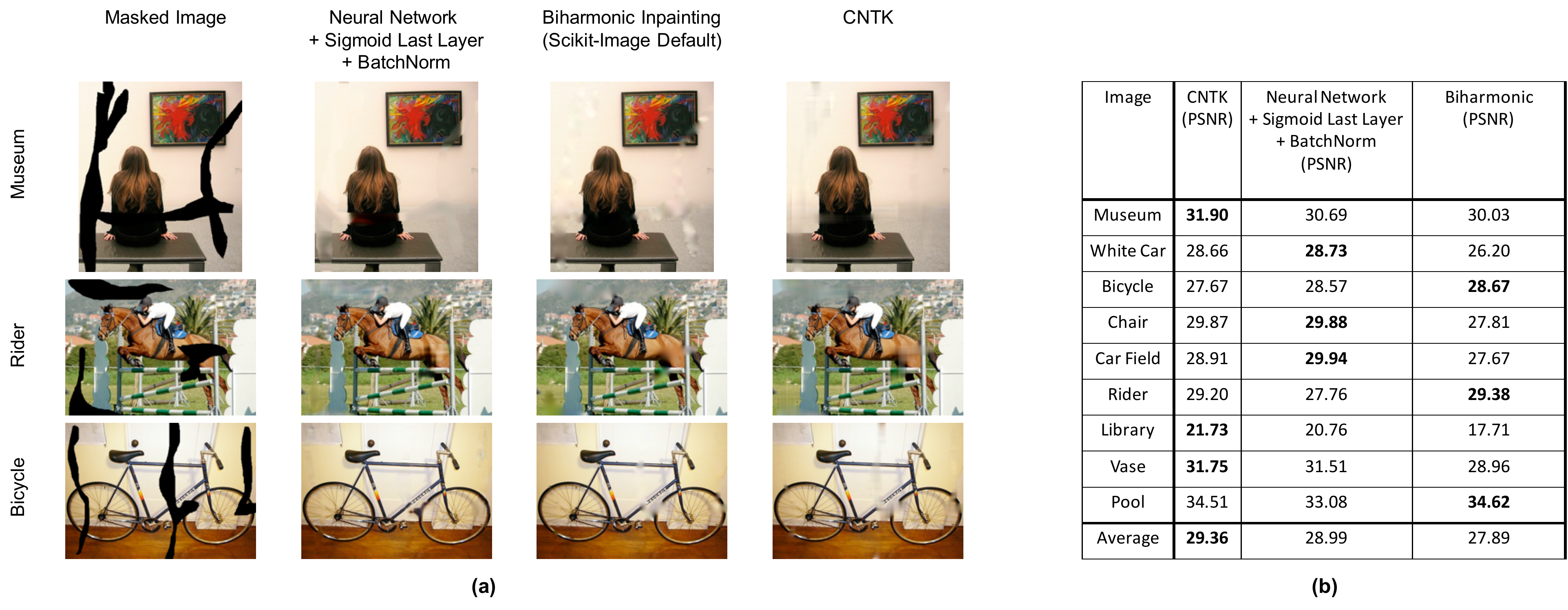} \caption{Large hole inpainting using (i) the CNTK, (ii) neural networks with sigmoid last layer and batch normalization layers that are trained with Adam, and (iii) biharmonic functions. (a) Qualitative comparison of inpainting results across the three methods.  Results for all images are provided in Appendix Fig.~\ref{fig: SI Large Hole Inpainting Images}. (b) Comparison of peak signal-to-noise ratio (PSNR) across 3 methods with the CNTK providing the highest average PSNR.  Runtime and  structural similarity index measure (SSIM) for the three methods are provided in Appendix Fig.~\ref{fig: SI Large Hole Inpainting Tables}.}
    \label{fig: large hole inpainting comparison table}
\end{figure*}

\section{Image Inpainting and Reconstruction with the CNTK} 
\label{sec: Image Inpainting Experiments}
We now utilize the results of the previous section to perform large hole image inpainting and reconstruction.  As illustrated in Figs.~\ref{fig: Overview Schematic}c and \ref{fig: Overview Schematic}d, large hole inpainting involves imputing a large contiguous region in an image while image reconstruction involves imputing random missing pixels in an image.  Recent work \cite{DeepImagePrior} demonstrated that using convolutional neural networks with downsampling and upsampling layers to impute the missing pixels in images leads to competitive results for these applications.  

The methods from \cite{DeepImagePrior} are a special case of our framework in Eq.~[\ref{eq: Deep Matrix Factorization Framework}]; namely using convolutional layers and letting the \prior, $Z$, be a tensor with i.i.d.~uniform random entries.  Thus, we can use our framework for performing image completion tasks, and instead of training deep networks, we can simply solve kernel regression with the CNTK.  We will demonstrate that this gives rise to a simple, fast, flexible, and competitive alternative to training deep networks for high resolution image completion problems.  Moreover, we will demonstrate that our framework can be used to identify the role of architecture and \prior on image completion problems and aid in identifying effective architectures and \priors.  

\subsection*{Application 1: Large Hole Inpainting with the CNTK} 

We utilize the CNTK for large hole inpainting tasks from \cite{BayesianDeepImagePrior,DeepImagePrior}.  We compute the CNTK for the architecture used in \cite{BayesianDeepImagePrior} with 6 downsampling and nearest neighbor upsampling layers for the \prior $Z$ with i.i.d.~entries $Z_{\ell,i,j} \sim U[0, .1]$, where $c \in \mathbb{Z}_+$ and $i, j \in [m] \times [n]$.  We compute the CNTK on $128 \times 128$ resolution images and then expand it to the CNTK for high resolution images via our expansion technique in Section \ref{sec: Image Inpainting CNTK}.  We compare our method against neural networks of the same architecture using the training procedures from \cite{DeepImagePrior, BayesianDeepImagePrior} (see Appendix \ref{appendix: N} for details).  We also compare our method against inpainting with biharmonic functions \cite{BiharmonicInpainting}, which is currently the default inpainting method in scikit-image \cite{ScikitImage}.  

% , \textcolor{red}{do we need to show that these things actually make the algorithms work better?} but we include batch normalization layers, sigmoid activation on the last layer, and train with the Adam \cite{Adam} optimizer, as is done in

% We train our neural network with Adam \cite{Adam} and a learning rate of $0.01$ as is done in \cite{DeepImagePrior, BayesianDeepImagePrior}.\footnote{SGD in \cite{BayesianDeepImagePrior} seems to actually refer to Adam based on the provided code.}   We train our neural networks for 1000 epochs and perform 10 epochs of kernel regression with EigenPro for the CNTK.  

Figure \ref{fig: large hole inpainting comparison table}a shows examples of the resulting reconstructions, and Figure \ref{fig: large hole inpainting comparison table}b shows the peak signal-to-noise ratio (PSNR) across all methods.  Our method on average outperforms both inpainting with finite width neural networks and inpainting with biharmonic functions.\footnote{While the PSNR values for these images are also presented in \cite{BayesianDeepImagePrior}, they appear to be computed without replacement of the observed pixel values.  We re-ran these experiments with replacement for fair comparison with biharmonic inpainting.}  In Appendix Fig.~\ref{fig: SI Large Hole Inpainting Tables}, we show that our method also outperforms the other methods in terms of structural similarity index measure (SSIM), and that the runtime is comparable (within 2 minutes on average) across all methods in this setting.  The reconstructions across all images and methods are provided in Appendix Fig.~\ref{fig: SI Large Hole Inpainting Images}.

\subsection*{Application 2: Image Reconstruction with the CNTK} We next analyze the performance of the CNTK on the image reconstruction tasks considered in \cite{DeepImagePrior}.  While the networks considered in \cite{BayesianDeepImagePrior, DeepImagePrior} make use of skip connections for image reconstruction, we only consider architectures without skip connections for which we can derive the CNTK exactly (see Appendix \ref{appendix: N} for details).  We again compare the CNTK to neural networks of the same architecture and to biharmonic inpainting.  For this comparison, we use networks with 128 filters per layer, as is done in \cite{DeepImagePrior, BayesianDeepImagePrior}.   In Appendix Fig.~\ref{fig: SI Image Reconstruction Table}, we show that our model performs comparably to inpainting with biharmonic functions and outperforms neural networks of the same architecture.  In Appendix Fig.~\ref{fig: SI Image Reconstruction Table}, we additionally show that our method performs comparably to biharmonic inpainting in terms of SSIM and that our method is up to 10 times faster than using small width neural networks on the same hardware.  While our method performs comparably to inpainting with biharmonic functions in this application, our framework is more flexible, since we can adjust architecture and \prior, and it outperforms  inpainting with biharmonic functions for the problem of large hole inpainting (see above).  Since methods such as Adam with Langevin dynamics \cite{BayesianDeepImagePrior} have enabled performance boosts for neural networks (see Appendix Fig.~\ref{fig: SI Large Hole Inpainting Tables} \& \ref{fig: SI Large Hole Inpainting Images}), an interesting direction for future work could be to incorporate such techniques for image completion applications using the CNTK.

\begin{figure*}
    \centering
    \includegraphics[width=\textwidth]{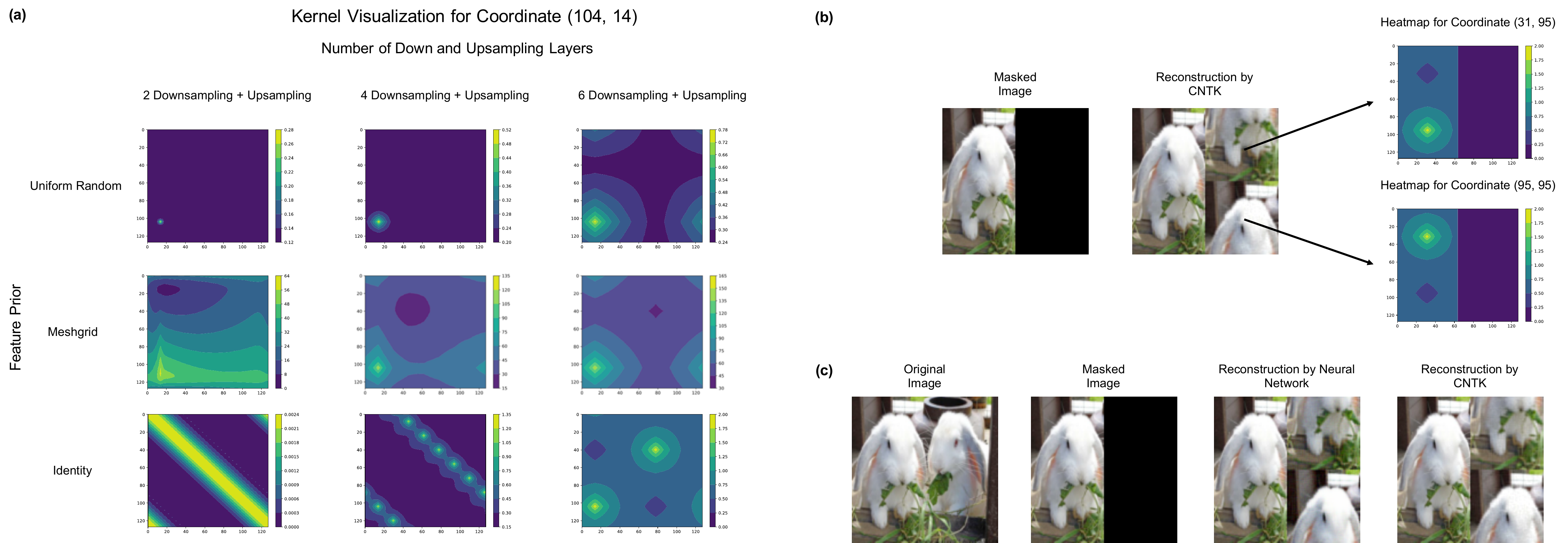} \caption{We use the CNTK to understand the impact of architecture and input on image inpainting. (a) Heatmap visualizations of the CNTK when varying the number of downsampling/upsampling layers and input.  The visualization makes clear that the uniform random \prior, unlike other \priors,  results in kernels that use the region surrounding a missing pixel value for imputation regardless of the number of downsampling layers.  (b) The heatmap visualizations of the CNTK make transparent which observed pixels are being used to inpaint a given missing pixel when using the identity \prior. (c) A comparison between inpainting a $128 \times 128$ resolution image of a rabbit with a finite width neural network and with the CNTK when the \prior is the identity. The CNTK is able to accurately predict the unexpected behavior of the neural network. }
    \label{fig:Heatmap Visualization}
\end{figure*}

\subsection*{Using Our Framework to Select \PriorAllCaps and Architecture for Image Completion} 
\label{sec: Impact of architecture on image inpainting}

%While the CNTK of U-Net architectures and \prior with i.i.d. entries is effective for image completion tasks, it remains unclear why these achitectures and \priors are effective.  Moreover, it remains unclear whether other standard \priors (e.g. the identity matrix) or architectures are also effective for image completion tasks.  
In the following, we demonstrate that our framework provides a theoretical underpinning for understanding how a given architecture and \prior influence image completion.  In particular, we use our framework to explain why the uniform random \prior and architectures with downsampling and upsampling layers are effective for image completion while other \priors such as the identity \prior are ineffective for this application. 

The key observation enabling such interpretability is that for kernel methods, every prediction (a missing pixel value) is a linear combination of training examples (observed pixel values).  Hence, for each imputed pixel, the CNTK can be used to provide a heatmap describing which observed pixels were most heavily weighted in the linear combination.  In order to generate such heatmaps, we reshape the CNTK into a 4 dimensional tensor.  Namely, given a CNTK $K \in \mathbb{R}^{mn \times mn}$, we reshape $K$ to a tensor $K_T \in \mathbb{R}^{m \times n \times m \times n}$ where $K(M_{ij}, M_{i'j'}) = K_T(i, j, i', j')$.  To generate a heatmap for a given a coordinate $(i,j)$, we visualize the matrix $K_T(i, j, :, :) \in \mathbb{R}^{m \times n}$.  This visualization allows us to decipher how architecture and \prior change the resulting imputation from a neural network.

% \textcolor{red}{this paragraph is very repetitive with the previous one; please consolidate; no need for a subsection.}\subsubsection*{Visualizing Heatmaps of the CNTK to Understand the Role of Architecture and \PriorAllCaps} Given that we are solving kernel regression with the CNTK, each imputed pixel must be a linear combination of observed pixels.  Hence, for each imputed pixel, the CNTK can be used to provide a heatmap describing which observed pixels were most heavily weighted in the linear combination.  In order to generate such heatmaps, we reshape the CNTK into a 4 dimensional tensor.  Namely, given a CNTK $K \in \mathbb{R}^{mn \times mn}$, we reshape $K$ to a tensor $K_T \in \mathbb{R}^{m \times n \times m \times n}$ where $K(M_{ij}, M_{i'j'}) = K_T(i, j, i', j')$.  To generate a heatmap for a given a coordinate $(i,j)$, we visualize the matrix $K_T(i, j, :, :) \in \mathbb{R}^{m \times n}$.  This visualization allows us to decipher how architecture and \prior change the resulting imputation from a neural network.

\subsubsection*{The Uniform Random \PriorAllCaps and \textcolor{black}{Modern} Architectures are Effective for Image Completion} In Fig.~\ref{fig:Heatmap Visualization}a, we visualize the kernel values $K(104, 14, :, :)$ computed for a $128 \times 128$ image when varying the number of down and upsampling layers and as well as the \prior $Z$.  Namely, we consider the cases where $Z$ is the identity, the meshgrid from \cite{DeepImagePrior}, or the uniform random tensor used in large hole inpainting experiments of \cite{DeepImagePrior}.  A key observation is that the kernel values for the uniform random \prior are highest around the coordinate of interest regardless of the amount of down and upsampling, which is in stark contrast to other \priors.\footnote{When there are no downsampling and upsampling layers, this follows immediately from Theorem \ref{theorem: Coordinate only CNTK}.}  This implies that neighboring pixels are most heavily used when imputing using the uniform random \prior (see Appendix Fig.~\ref{fig: SI Norms Kernels} for additional visualizations). Moreover, when using the uniform random \prior, the amount of down and upsampling increase (by powers of 2) the size of the region considered for imputation (see the first row of Fig.~\ref{fig:Heatmap Visualization}a).  These heatmaps identify the minimum amount of downsampling necessary for large hole inpainting: if there is an $m \times m$ region of missing pixels ($m \geq 1$), we need at least $\lfloor \log_2 (m + 1) \rfloor$ layers of downsampling to ensure that no pixel is filled in as an average of all other pixels.    This result explains the observation from \cite{DeepImagePrior}, which showed that using neural networks with four or fewer downsampling and upsampling layers led to worse large hole inpainting performance on images with large missing regions.

% \textcolor{red}{repetitive sentence, please consolidate; 2 paragraphs can be made into 1.} We now use the CNTK heatmaps to explain why the uniform random \prior and U-Net architectures with downsampling and upsampling layers are effective for image completion tasks.  In Fig.~\ref{fig:Heatmap Visualization}a, we visualize the kernel values $K(104, 14, :, :)$ computed for a $128 \times 128$ image when varying the number of down and upsampling layers and $Z$.  Namely, we consider the cases where $Z$ is the identity, the meshgrid from \cite{DeepImagePrior}, or the uniform random tensor used in large hole inpainting experiments of \cite{DeepImagePrior}.  

\subsubsection*{The Identity \PriorAllCaps is Ineffective for Image Completion}
The standard \prior for matrix completion is given by choosing $Z$ to be the identity matrix \cite{DeepMatrixFactorization, MatrixCompletionNuclearNorm, MatrixFactorizationNIPS2017}.  As shown in Fig.~\ref{fig:Heatmap Visualization}a, unlike the uniform random \prior, the identity \prior uses pixel observations from non-local regions for completion.  Thus, we expect this \prior to be ineffective for image completion tasks.

% We now demonstrate that heatmaps from our framework are useful for identifying \priors and architectures that are effective for image completion.  Additionally, we show that our framework accurately predicts the behavior neural networks for image completion and thus, can be used as a rapid prototyping tool for identifying effective neural network architectures for image completion.  

% From Fig.~\ref{fig:Heatmap Visualization}a, we observe that the CNTK for the uniform random \prior makes use of nearby pixel values to impute a missing pixel value.  Hence, this \prior makes use of local image context for image completion.  On the other hand, when the \prior is the identity matrix, i.e. the standard \prior used for matrix completion \cite{DeepMatrixFactorization, MatrixCompletionNuclearNorm, MatrixFactorizationNIPS2017}, Fig.~\ref{fig:Heatmap Visualization}a demonstrates that the resulting imputation uses pixel observations from non-local regions for completion.  \textcolor{blue}{We thus expect this \prior to perform poorly for images that have differing non-local regions.} 

Fig.~\ref{fig:Heatmap Visualization}b %, we demonstrate that using the identity \prior is ineffective for image inpainting.  In particular, 
shows the result of using the CNTK for a network with 6 downsampling and upsampling layers and the identity \prior to impute a $128 \times 128$ rabbit image.  The identity \prior visually appears to translate observed pixels from a non-local region to perform imputation.  The regions that are being translated are precisely those given by the corresponding heatmaps, e.g. the upper right quadrant is imputed using the lower left quadrant in Fig.~\ref{fig:Heatmap Visualization}b. 

We note that our framework accurately predicts the behavior of finite width neural networks used for image inpainting.  In Fig.~\ref{fig:Heatmap Visualization}c, we show the result of using a neural network with 6 downsampling and upsampling layers,  sigmoid activation on the last layer, and identity  \prior.  We observe that the neural network completes the image by translating observed pixels similarly to the imputation provided by the corresponding CNTK.  This example highlights the power of using our framework for rapidly prototyping \priors and architectures for image inpainting tasks.

\section{Discussion}

% \misha{I think we need to be more systematic here -- why our framework is conceptually simple and also computationally simple}

%  We emphasize key conceptual and computational aspects of our framework below. 
% \begin{enumerate}
%     \item  Utilizing the fact that training infinite width neural networks (under mild assumptions) are equivalent to kernel methods, we derived a closed form for the  
%     \item \textbf{Flexibility through \PriorAllCaps.}  
% \end{enumerate}

In this work, we presented a simple, fast, and flexible framework for matrix completion using the infinite width limit of neural networks, i.e. the neural tangent kernel (NTK).  Below, we highlight the aspects of our framework that enable such simplicity, speed, and flexibility.  

\begin{itemize}
    \item \textbf{Simple.} Our framework is conceptually simple since we are using kernels to learn a map from features of coordinates, $(i,j)$, to entries in the target matrix,  $Y_{i,j}$.  Our framework is computationally simple since solving kernel regression involves solving a linear system of equations.  
    \item \textbf{Fast.}  Our framework is naturally fast when using the NTK of fully connected networks for matrix completion due to the simple closed form of the kernel (Theorem \ref{theorem: d layer relu ntk}).  We develop a memory and runtime efficient algorithm to compute and use the NTK of convolutional networks (the CNTK) for matrix completion (Section \ref{sec: Image Inpainting CNTK}).  
    \item \textbf{Flexible.}  Our framework is easily adapted to various applications by the choice of the \prior, thereby making our framework flexible.  Moreover, we provided a principled approach for selecting the \prior by establishing a connection with semi-supervised learning (Theorems \ref{theorem: d layer relu ntk}, \ref{theorem: Coordinate only CNTK}) and providing a visualization of the effect of the \prior (Section 4).  
\end{itemize}

The simplicity and speed of our framework is illustrated by the fact that many of our results (including inpainting high resolution images) can be run on a CPU and even on a laptop (see Materials \& Methods for a link to our code).  We demonstrated that our framework is flexible by using it to achieve competitive results for virtual drug screening (Section \ref{sec: Drug Imputation with the NTK}) and image inpainting/reconstruction (Section \ref{sec: Image Inpainting Experiments}).  We envision that our work provides a simple and accessible framework for producing strong baselines for several matrix completion applications.  We conclude with a discussion of possible future extensions and applications.

\subsubsection*{Future Applications of Our Framework}  In this work, we demonstrated the flexibility of our framework by constructing \priors for two different applications, namely virtual drug screening and image completion.  An interesting future direction is the extension of our framework to other modalities such as tensors, video, or audio data.  For example, by using a \prior that captures the structure of coordinates in 3D  images, we could  apply our framework to impute missing regions in three-dimensional data.  

\subsubsection*{Efficient Computation of the CNTK} In classification and regression settings, a major hindrance for using the CNTK in practice is the computational complexity in computing the kernel for a large image dataset.  In this work, we presented an expansion technique to efficiently compute and store the exact CNTK for inpainting high resolution images, which was previously considered infeasible \cite{DenoisingNTK, BayesianDeepImagePrior}.  By understanding the properties of the CNTK that make it effective for image problems, we envision that similar techniques could be applied to produce efficient kernel machines for image classification.  

\subsubsection*{Developing Techniques to Improve the Performance of the NTK}   While a large number of techniques such as skip connections, batch normalization, etc. have been developed to augment the performance of neural networks, such techniques have yet to be adapted to improve the performance of kernels. The simplicity and effectiveness of the NTK and CNTK based on simple architectures considered in this work motivates the development of techniques to further boost the performance of the NTK and kernel methods in general.  

\section*{Materials and Methods}
For solving kernel regression with the NTK, we use the direct linear system solver from \cite{numpy1} when the number of equations is fewer than 30,000, and we use EigenPro \cite{EigenProGPU, EigenPro} otherwise.  For training neural networks, we use the PyTorch library \cite{PyTorch}.  All methods requiring a GPU are run on a single NVIDIA Titan RTX GPU.  Our experiments are run on a shared server with 4 Titan RTX GPUs, 128GB CPU RAM, and 64 threads.  
% we note that many experiments (including those for high resolution image inpainting) can be run on a laptop.  We provide a link to a Jupyter notebook \textcolor{red}{TODO}

For the virtual drug screening experiments, we use the subset of the CMap dataset \cite{CMap} provided in \cite{SontagDrugImputation}.  A detailed description of all the methods (including random seeds and hyperparameters for DNPP and FaLRTC) and evaluation metrics for the virtual drug screening experiments is provided in Appendices \ref{appendix: C}-\ref{appendix: H}.  A description of the t-test used for determining the significance of our results for virtual drug screening is presented in Appendix \ref{appendix: I}.  We provide code to replicate our results for the virtual drug screening experiments with the NTK, DNPP, FaLRTC, and mean over cell type in the footnote below\footnote{ \href{https://github.com/uhlerlab/ntk_matrix_completion}{https://github.com/uhlerlab/ntk\_matrix\_completion}}.  We use the codebase from \cite{SontagDrugImputation} for performing imputation with FaLRTC.  

For the image completion applications, we use the datasets from \cite{BayesianDeepImagePrior, DeepImagePrior}.  The rabbit image used in Fig.~\ref{fig:Heatmap Visualization} is from \cite{ImageNet} and is provided in our codebase (linked above). For the neural network and NTK methods used in our image inpainting and reconstruction experiments, we provide a description of all architectures and training hyperparameters in Appendix \ref{appendix: N}.   

We provide a library for computing and using the CNTK for image inpainting and reconstruction applications in the codebase linked above.  Our library lets the user define a custom neural network (similarly to network definitions in PyTorch), and then provides a function to compute the CNTK from the given architecture.  Our method for computing the CNTK runs entirely on the CPU, and we enable parallelization across CPU threads.  Our library includes functions for computing the CNTK for networks with nearest neighbor and bilinear upsampling layers, which are not readily available in the Neural Tangents library \cite{NeuralTangentsGoogle}.  We additionally provide functions to solve kernel regression using the CNTK via a linear system solver or EigenPro.  A full description of the library and an example of how to use our library for image inpainting is provided in Jupyter notebooks in our linked code.  We additionally release several pre-computed kernels that can be used for high resolution inpainting and reconstruction.

\section*{Acknowledgements}

A.R., G.S., and C.U.~were partially supported by NSF (DMS-1651995), ONR (N00014-17-1-2147 and N00014-18-1-2765), the MIT-IBM Watson AI Lab, \textcolor{black}{the Eric and Wendy Schmidt Center at the Broad Institute,} and a Simons Investigator Award (to C.U.). M.B.~acknowledges support from NSF  IIS-1815697 and  NSF DMS-2031883/Simons Foundation  Award 814639.

%The Titan Xp used for this research was donated by the NVIDIA Corporation. 

\bibliographystyle{abbrv}
\bibliography{references}

\clearpage

\appendix

\section*{Appendix}

\section{Preliminaries on the NTK}
\label{appendix: A}

In this section, we review notation from prior literature on the NTK \cite{NTKJacot} that will be used throughout this work.  In particular, we review how the NTK can be computed in closed form using dual activation functions~\cite{DualActivation}.  We start by providing the definition of the NTK.

\begin{definition}[NTK] Let $f(\mathbf{W} ; x): \mathbb{R}^{p} \times \mathbb{R}^{d} \to \mathbb{R}$ denote a neural network with parameters $\mathbf{W}$.  The \textbf{neural tangent kernel}, $K:\mathbb{R}^{d} \times \mathbb{R}^{d} \to \mathbb{R}$, is a symmetric, continuous, positive definite function given by: 
\begin{align*}
    K(x, x') = \langle \nabla_{\mathbf{W}} f(\mathbf{W}^{(0)} ; x), \nabla_{\mathbf{W}} f(\mathbf{W}^{(0)} ; x') \rangle,
\end{align*}
where $\mathbf{W}^{(0)} \in \mathbb{R}^{p}$ denotes the parameters at initialization.  
\end{definition}

In this section, we consider fully connected networks of the following form: 
\begin{align}
\label{eq: Fully connected network}
    f(\mathbf{W}; x) = W^{(L)} \frac{c}{\sqrt{k_{L}}} \phi\left( W^{(L-1)} \frac{c}{\sqrt{k_{L-1}}} \phi \left( \ldots \frac{c}{\sqrt{k_1}} \phi \left(W^{(1)} x \right) \ldots  \right) \right), 
\end{align}
where $\mathbf{W} = \{W^{(i)}\}_{i=1}^{L}$ with $W^{(i)} \in \mathbb{R}^{k_{i} \times k_{i-1}}$ and $k_0 = d, k_{L} = 1$; $\phi: \mathbb{R} \to \mathbb{R}$ is an elementwise Lipschitz nonlinearity; and $c$ is a constant.  The key finding of \cite{NTKJacot} is that when $\mathbf{W}_i \overset{i.i.d.}{\sim} \mathcal{N}(0, 1)$, then as $k_1, k_2, \ldots k_L \to \infty$, $K_L(x, x')$ converges in probability to a deterministic kernel that does not change through training. Thus, solving kernel ridge-less regression with kernel $K_L$ is equivalent to the solution given by training the neural network.  We present the case for fully connected networks from \cite{NTKJacot} below, but will also be using the results for convolutional networks from \cite{CNTKArora} later on.  

\begin{theorem*}
\label{thm: NTK Recursion}
Let $f: \mathbb{R}^{d} \to \mathbb{R}$ be a neural network defined in Eq.~[\ref{eq: Fully connected network}].  As $k_1, k_2, \ldots k_L \to \infty$, then $K_L(x, x')$ converges in probability to a deterministic kernel given by the following recurrences in $\Sigma_i, \dot{\Sigma}_i, K_i$:
\begin{align*}
    K_0(x, x') &= \Sigma_0(x, x') =  x^T x', \\
    K_L(x, x') &= \Sigma_L(x, x') + K_{L-1}(x, x') \Sigma_{L-1}'(\Sigma_{L-1}(x, x')),\\
    \Sigma_L(x, x') &=  c^2 \mathbb{E}_{(u,v) \sim \mathcal{N}(\mathbf{0}, \Lambda_{L-1}(x, x') )} [\phi(u) \phi(v)], \\
    \dot{\Sigma}_L(x, x') &= c^2 \mathbb{E}_{(u,v) \sim \mathcal{N}(\mathbf{0}, \Lambda_{L-1}(x, x') )} [\phi'(u) \phi'(v)], \\
    \Lambda_L(x, x') &= 
    \begin{bmatrix} 
    \Sigma_{L-1}(x, x)  & \Sigma_{L-1}(x, x') \\
    \Sigma_{L-1}(x', x)  & \Sigma_{L-1}(x', x')    
    \end{bmatrix}. 
\end{align*}
\end{theorem*}

\noindent \textbf{Dual Activations.} The expectations in the recurrences above can be simplified using the theory of dual activation functions studied in \cite{DualActivation}.  Let $\check{\phi}: [-1, 1] \to \mathbb{R}$ such that:
\begin{align}
\label{eq: Dual Activation}
    & \check{\phi}(\xi) = c^2 \mathbb{E}_{(u,v) \sim \mathcal{N}(\mathbf{0}, \Lambda) )} [\phi(u) \phi(v)], 
    ~~ \Lambda = \begin{bmatrix} 1 & \xi \\ \xi & 1 \end{bmatrix}, ~~ \frac{1}{\sqrt{2\pi}} \int_{\mathbb{R}} \phi(u)^2 \exp{\left(-\frac{u^2}{2}\right)} du = \frac{1}{c^2}.  
\end{align}
The map $\mathcal{F}$ such that $\mathcal{F}(\phi) = \check{\phi}$ is an operator mapping from activation functions to positive definite functions\footnote{The map $\mathcal{F}$ is more precisely from the Hilbert space $L^2(\mu)$ with $\mu$ the Gaussian measure to the space of positive definite functions.  The factor $c$ is selected so that $\|\phi\|_{L^2(\mu)} = 1$.}, and $\check{\phi}$ is referred to as the \emph{dual activation} \cite{DualActivation}.  The scaling factor $c$ in Theorem \ref{thm: NTK Recursion} is typically selected to satisfy the integral equation in Eq.~[\ref{eq: Dual Activation}].  As an example, when $\phi$ is the ReLU, the integral is just $\frac{1}{2}$ times the second moment of the standard Gaussian distribution.  Hence, $c = \sqrt{2}$ for the ReLU.  The recurrence relation for the NTK can be drastically simplified for homogenous nonlinearities for which the dual activation has a closed form.  As shown in prior work \cite{CosineKernel, LeakyReLUNTK}, this is the case for the commonly used ReLU and LeakyReLU nonlinearities.  In particular, the dual activation function for ReLU is well known \cite{CosineKernel}, and we next present its form (with its derivative): 
\begin{lemma*}
The dual activation $\hat{\phi}: [-1, 1] \to \mathbb{R}$ of the ReLU is:
\begin{align}
\label{eq: ReLU Dual}
    \check{\phi}(\xi) &= \frac{1}{\pi}(\xi (\pi - \cos^{-1}(\xi)) + \sqrt{1 - \xi^2}) \\
    \frac{d\check{\phi}(\xi)}{d\xi} &= \frac{1}{\pi} (\pi - \cos^{-1}(\xi)) \nonumber
\end{align}
\end{lemma*}

As shown in \cite{InductiveBiasNTK, LaplaceAndNTK1}, the NTK recursion for ReLU networks can be simplified using the dual activation.  We provide this known simplification below for completeness.  

\begin{prop*}
Let $f: \mathbb{R}^{d} \to \mathbb{R}$ be a neural network defined in Eq.~\ref{eq: Fully connected network}.  Let $\phi$ be the ReLU activation and let $c = \sqrt{2}$.  As $k_1, k_2, \ldots k_L \to \infty$, then $K_L(x, x')$ converges in probability to a deterministic kernel given by the following recurrences in $\Sigma_i, K_i$:
\begin{align*}
    K_0(x, x') &= \Sigma_0(x, x') =  x^T x',  \\
    \Sigma_L(x, x') &= N_{L-1}(x, x') \check{\phi}\left(\frac{\Sigma_{L-1}(x, x')}{N_{L-1}(x, x')}\right),\\
    K_L(x, x') &= \Sigma_L(x, x') + K_{L-1}(x, x') \frac{d\check{\phi}}{d\xi} \left(\frac{\Sigma_{L-1}(x, x')}{N_{L-1}(x, x')}\right), \\
    N_{L-1}(x, x') &= \sqrt{\Sigma_{L-1}(x,x) \Sigma_{L-1}(x',x')}.
\end{align*}
\end{prop*}

This proposition follows from using the change of variables $u = \sqrt{\Sigma_L(x, x)} \tilde{u}$ and $v = \sqrt{\Sigma_L(x', x')} \tilde{v}$ and the homogeneity of ReLU when computing $c^2 \mathbb{E}_{(u,v) \sim \mathcal{N}(\mathbf{0}, \Lambda_{L-1}(x, x') )} [\phi(u) \phi(v)]$ using integration \cite{InductiveBiasNTK}.  In this work, we will use the dual activation for both ReLU and LeakyReLU \cite{LeakyReLU} in order to match popular deep learning architectures as closely as possible.  The derivation for the dual activation for LeakyReLU is provided in \cite{LeakyReLUNTK}.

\section{Proofs for Matrix Completion with the NTK}
\label{appendix: B}
\label{appendix: matrix completion with the NTK}
We present the statement of Theorem 1 for a general homogeneous (degree 1), Lipschitz nonlinearity below and then present the proof.  We again note that ReLU and LeakyReLU are commonly used nonlinearities that satisfy these conditions.  The results are easily extended to homogeneous nonlinearities of arbitrary degree and for \priors that have columns with arbitrary norm.  

\begin{theorem*}
Assume $Z = \{z^{(i)}\}_{i=1}^{n} \in \mathbb{R}^{p \times n}$, where each column is normalized with $\|z^{(i)} \|_2 = 1$.  Let $f_Z(\mathbf{W})$ be a $d$ layer fully connected network with Lipschitz nonlinearity $\phi$ that is homogeneous of degree 1 and $c= \| \phi \|_{L^2(\mu)}^{-1}$ where $L^2(\mu)$ is the Hilbert space of square Lebesgue integrable functions under Gaussian measure.  Then as layer widths $k_1 \to \infty, k_2 \to \infty, \ldots, k_{d-1} \to \infty$, the NTK for matrix completion with $f_Z(\mathbf{W})$ is given by 
\begin{align*}
    K_d(M_{ij}, M_{i'j'}) =  \begin{cases}
    \kappa_d({z^{(j)}}^T z^{(j')}) ~~~ \textrm{if } i = i' \\
    0 ~~~~~~~~~~~~~~~~~~ \textrm{if } i \neq i' 
    \end{cases},
\end{align*}
where $\kappa_d(\xi) = \check{\phi}^{(d)}(\xi) + \kappa_{d-1}(\xi) \frac{d\check{\phi}}{d\xi} (\check{\phi}^{(d-1)}(\xi))$, and $\check{\phi}^{(k)}(\xi) = \check{\phi}( \check{\phi}^{(k-1)}(\xi))$ for $k \geq 1$ and $\check{\phi}^{(0)}(\xi) = \xi$.
\end{theorem*}

\begin{proof}
We proceed by induction and present the case for $d=1$ first.  Namely, we define $g_Z(M)$ as follows: 
\begin{align*}
    g_Z(M) = tr(M^T A \phi(B Z)),
\end{align*}
where $A \in \mathbb{R}^{m \times k}, B \in \mathbb{R}^{k \times p}, Z \in \mathbb{R}^{p \times n}$.  To compute the kernel, we compute $\frac{\partial g_Z(M)}{\partial A_{\alpha, \beta}}, \frac{\partial g(M)}{\partial B_{\alpha, \beta}}$ directly.  We begin by expanding the matrix products in $g_Z(M)$.  For a matrix $U$, we let $U_{i,:}$ denote row $i$ of $U$ and $U_{:,i}$ denote column $i$ of $U$. Note that
\begin{align*}
    g_Z(M) &= tr\left(M^T A \frac{c}{\sqrt{k}}\phi(B Z)\right)\\
    &= \frac{c}{\sqrt{k}} tr\left(M^T A 
    \begin{bmatrix}
    \phi (B_{1,:} Z_{:,1}) & \ldots & \phi (B_{1,:} Z_{:,n}) \\
    \vdots & \ldots & \vdots \\
    \phi (B_{k,:} Z_{:,1}) & \ldots & \phi (B_{k:} Z_{:,n})
    \end{bmatrix} 
    \right) \\
    &= \frac{c}{\sqrt{k}} tr\left(M^T 
    \begin{bmatrix}
    \sum_{a=1}^{k} A_{1,a} \phi (B_{a,:} X_{:,1}) & \ldots & \sum_{a=1}^{k} A_{1,a} \phi (B_{a,:} Z_{:,n}) \\
    \vdots & \ldots & \vdots \\
    \sum_{a=1}^{k} A_{m,a} \phi (B_{a,:} X_{:,1}) & \ldots & \sum_{a=1}^{k} A_{m,a} \phi (B_{a,:} Z_{:,n})
    \end{bmatrix} 
    \right)\\
    &= \frac{c}{\sqrt{k}} \sum_{i=1}^{m}\sum_{j=1}^{n} M_{i,j} \sum_{a=1}^{k} A_{i,a} \phi(B_{a,:} Z_{:, j}).
\end{align*}

We thus have that
\begin{align*}
    \frac{\partial g_Z(M)}{\partial A_{\alpha, \beta}} &= \frac{c}{\sqrt{k}} \sum_{j=1}^{n} M_{\alpha, j} \phi (B_{\beta, :} Z_{:, j}), \\
    \frac{\partial g_Z(M)}{\partial B_{\alpha, \beta}} &= \frac{c}{\sqrt{k}} \sum_{i=1}^{m}\sum_{j=1}^{n} M_{i,j}  A_{i,\alpha} \phi(B_{\alpha,:} Z_{:, j}) Z_{\beta, j}.\\
\end{align*}

The NTK is given by: 
\begin{align*}
    K_1(M, \tilde{M}) &= \langle\nabla g_Z(M), \nabla g_Z(M')\rangle \\
    &=\sum_{\alpha=1}^{m} \sum_{\beta=1}^{k}\frac{\partial g_Z(M)}{\partial A_{\alpha, \beta}}\cdot \frac{\partial g_Z(M')}{\partial A_{\alpha, \beta}} + \sum_{\alpha=1}^{k} \sum_{\beta=1}^{p} \frac{\partial g_Z(M)}{\partial B_{\alpha, \beta}}\cdot \frac{\partial g_Z(M')}{\partial B_{\alpha, \beta}}.
\end{align*}
To simplify the computation, we note that we will only ever need the gradient at indicator matrices $M_{ij}$ and $M_{i'j'}$.  Moreover, from the formula for the partial derivatives, we conclude that 
\begin{align*}
    \frac{\partial g_Z(M_{ij})}{\partial A_{\alpha, \beta}} &= \begin{cases}
    0 & \text{if $\alpha \neq i$} \\
    \frac{c}{\sqrt{k}} \phi(B_{\beta,:} Z_{:, j}) & \text{otherwise}
    \end{cases}, \\
    \frac{\partial g_Z(M_{ij})}{\partial B_{\alpha, \beta}} &= \frac{c}{\sqrt{k}} A_{i, \alpha} \phi(B_{\alpha, :} Z_{:, j}) Z_{\beta, j}.\\
\end{align*}
Thus, we can simplify the NTK as follows: 
\begin{align*}
    \lim_{k \to \infty} K_1(M_{ij}, M_{i'j'}) &= \lim_{k \to \infty}  \frac{c^2}{k}\sum_{\beta=1}^{k} \phi(B_{\beta, :} Z_{:, j}) \phi(B_{\beta, :} Z_{:, j'}) \mathbf{1}_{i = i'} \\
    &~~~~~~~~~~ + \frac{c^2}{k}\sum_{\alpha=1}^{k} A_{i, \alpha} A_{i', \alpha} \sum_{\beta=1}^{p} \phi'(B_{\beta, :} Z_{:, j})  \phi'(B_{\beta, :} Z_{:, j'})  Z_{\beta, j} Z_{\beta, j'} \\
    & = \begin{cases}
    0 & i \neq i' \\
    \kappa_1 ({z^{(j)}}^T z^{(j')}) & i = i'
    \end{cases}, 
\end{align*}
which completes the base case.

For the inductive step, we assume that
\begin{align*}
    \lim_{k_{d-2} \to \infty} \ldots \lim_{k_{1} \to \infty}  K_{d-1}(M_{ij}, M_{i'j'}) &= \begin{cases}
    0 & i \neq i' \\
    \kappa_{d-1} ({z^{(j)}}^T z^{(j')}) & i = i'
    \end{cases}.
\end{align*}
We now show that $K_d(M_{ij}, M_{i'j'})$ has the desired form.  For this, we define:
\begin{align*}
    g_Z(M) = M^T A \frac{c}{\sqrt{k_{d-1}}}\phi( h_Z(\mathbf{W})),
\end{align*}
where $A \in \mathbb{R}^{m \times k_{d-1}}$ and $h_Z(\mathbf{W}): \mathbb{R}^{p \times n} \to \mathbb{R}^{k_{d-1} \times n}$ is a $d-1$ fully connected network operating on $Z$.  Following the computation for the $1$ layer case, we obtain 
\begin{align*}
    \frac{\partial g_Z(M)}{\partial A_{\alpha, \beta}} &= \frac{c}{\sqrt{k_{d-1}}} \sum_{j=1}^{n} M_{\alpha, j} \phi( h_Z(\mathbf{W}))_{\beta, j}, \\
    \frac{\partial g_Z(M)}{\partial \mathbf{W}_{\alpha, \beta}} &= \frac{c}{\sqrt{k_{d-1}}} \sum_{i=1}^{m} \sum_{j=1}^{n} M_{i,j}  \sum_{k=1}^{k_{d-1}} A_{i, k} \frac{\partial{\phi(h_Z(\mathbf{W}))_{k, j}}}{\partial \mathbf{W}_{\alpha, \beta}}.
\end{align*}
Now we consider the case of indicator matrices $M_{ij}, M_{i'j'}$.  For $M_{ij}$, we note that $ \frac{\partial g_Z(M)}{\partial A_{\alpha, \beta}}$ is only non-zero for the terms
\begin{align*}
     \frac{\partial g_Z(M_{ij})}{\partial A_{i, \beta}} = \frac{c}{\sqrt{k_{d-1}}} \phi( h_Z(\mathbf{W}))_{\beta, j}.
\end{align*}
Hence, if $i \neq i'$, we obtain that 
\begin{align*}
    \sum_{\alpha, \beta} \frac{\partial g_Z(M_{ij})}{\partial A_{\alpha, \beta}} \frac{\partial g_Z(M_{i'j'})}{\partial A_{\alpha, \beta}} = 0.
\end{align*}
Similarly, for $M_{ij}$, we have that 
\begin{align*}
    \frac{\partial g_Z(M_{ij})}{\partial \mathbf{W}_{\alpha, \beta}} &= \frac{c}{\sqrt{k_{d-1}}} \sum_{k=1}^{k_{d-1}} A_{i, k} \frac{\partial{\phi(h_Z(\mathbf{W}))_{k, j}}}{\partial \mathbf{W}_{\alpha, \beta}}.
\end{align*}
If $i \neq i'$, as $k_{d-1} \to \infty$, by law of large numbers: 
\begin{align*}
\sum_{\alpha, \beta}  \frac{\partial g_Z(M_{ij})}{\partial B_{\alpha, \beta}} \frac{\partial g_Z(M_{i'j'})}{\partial B_{\alpha, \beta}} \to 0.
\end{align*}
Thus, if $i \neq i'$, we conclude that $K_d(M_{ij}, M_{i'j'}) = 0$. On the other hand, if $i = i'$, then we have that 
\begin{align}
\label{eq: A term contributions to NTK}
    \sum_{\alpha, \beta}  \frac{\partial g_Z(M_{ij})}{\partial A_{\alpha, \beta}} \frac{\partial g_Z(M_{i'j'})}{\partial A_{\alpha, \beta}}  = \frac{c^2}{k_{d-1}} \sum_{k=1}^{k_{d-1}} \phi( h_Z(\mathbf{W}))_{k, j} \phi( h_Z(\mathbf{W}))_{k, j'}.
\end{align}
Similarly, if $i = i'$, we have that 
\begin{align*}
    \sum_{\alpha, \beta}  \frac{\partial g_Z(M_{ij})}{\partial B_{\alpha, \beta}} \frac{\partial g_Z(M_{i'j'})}{\partial B_{\alpha, \beta}} =  \frac{c^2}{k_{d-1}} \left(\sum_{k=1}^{k_{d-1}} A_{i, k} \phi'(h_Z(\mathbf{W}))_{k, j} \frac{\partial{h_Z(\mathbf{W})_{k, j}}}{\partial \mathbf{W}_{\alpha, \beta}}   \right) \left(\sum_{k=1}^{k_{d-1}} A_{i, k} \phi'(h_Z(\mathbf{W}))_{k, j'} \frac{\partial{h_Z(\mathbf{W})_{k, j'}}}{\partial \mathbf{W}_{\alpha, \beta}}\right)
\end{align*}
By the inductive hypothesis as $k_1, k_2, \ldots k_{d-2} \to \infty$, the above converges in probability to:
\begin{align}
\label{eq: B term contributions to NTK}
     \lim_{k_{d-2} \to \infty}  \ldots \lim_{k_{1} \to \infty}  \sum_{\alpha, \beta}  \frac{\partial g_Z(M_{ij})}{\partial B_{\alpha, \beta}} \frac{\partial g_Z(M_{i'j'})}{\partial B_{\alpha, \beta}} \to
     \frac{c^2}{k_{d-1}} \left(\sum_{k=1}^{k_{d-1}}  \phi'(h_Z(\mathbf{W}))_{k, j}  \phi'(h_Z(\mathbf{W}))_{k, j'} K_{d-1}(M_{k,j}, M_{k,j'}) \right). 
\end{align}
Therefore, when $i=i'$, adding Eqs.~[\ref{eq: A term contributions to NTK}] and [\ref{eq: B term contributions to NTK}] and applying the inductive hypothesis yields: 
\begin{align*}
    \lim_{k_{d-1} \to \infty}  \lim_{k_{d-2} \to \infty} \ldots \lim_{k_{1} \to \infty} K_d(M_{ij}, M_{i'j'}) = \check{\phi}(\check{\phi}^{(d-1)}(\langle z^{(j)}, z^{(j')}) \rangle) + K_{d-1} \frac{d \check{\phi}}{d \xi} \left( \check{\phi}^{(d-1)} (\langle z^{(j)}, z^{(j')} \rangle ))  \right),
\end{align*}
which completes the proof.
\end{proof}

We next provide an example showing how to compute the NTK for matrix completion.  

\begin{example}
Suppose we have: 
\begin{align*}
    Y = \begin{bmatrix} 
    y_{11} & .5 & .3 \\ 
    .1 & .2 & y_{23} \\
    .4 &  y_{32} & y_{33} \\
    \end{bmatrix}.
\end{align*}
Assuming we read off the observed entries of $Y$ in row major order and that $Z = I$ (the $3 \times 3$ identity matrix), then the NTK is given by:
\begin{align*}
    K = \begin{bmatrix} 
    \kappa(1) & \kappa(0) & 0 & 0 & 0 \\
    \kappa(0) & \kappa(1) & 0 & 0 & 0 \\
    0 & 0 & \kappa(1) & \kappa(0) & 0 \\
    0 & 0 & \kappa(0) & \kappa(1) & 0 \\
    0 & 0 & 0 & 0 & \kappa(1) 
    \end{bmatrix}. 
\end{align*}
The solution to kernel regression is given by: 
\begin{align*}
    \tilde{g}(M) = \begin{bmatrix} .5 & .3 & .1 & .2 & .4 \end{bmatrix} K^{-1} k(M),
\end{align*}
where $k(M)$ is the vector with entries $k(M_{ij}, M)$  for $(i,j) \in S$. As an example, for $M_{11}$, we have: 
\begin{align*}
    k(M_{11}) = \begin{bmatrix} \kappa(0) & \kappa(0) & 0 & 0 & 0  \end{bmatrix}^T.
\end{align*}
\end{example}

\vspace{0.2cm}
\textcolor{black}{This example demonstrates the key difference between the NTK of fully connected networks for matrix completion and the usual multivariate NTK: namely, the former corresponds to solving a \emph{separate} kernel regression problem for each row of the target matrix $Y$.}  By modifying the nonlinearity $\phi$ and the \prior in Theorem 1, our framework encapsulates a broad class of semi-supervised learning approaches for matrix completion.  We provide a nontrivial example below.  

\begin{example}[Semi-supervised Learning with the Graph Laplacian]
The following corollary to Theorem 2 proves that semi-supervised learning using the graph Laplacian operator from \cite{SemiSupervisedLearningManifolds} is a specific instance of matrix completion with the NTK of a linear neural network used for matrix completion.  

\begin{corollary*}
\label{corollary: graph Laplacian}
Let $X \in \mathbb{R}^{d \times n}$ denote a set of data points of which a subset $X_S \in \mathbb{R}^{d \times s}$ is labelled with labels $Y_S \in \mathbb{R}^{1 \times s}$.  Let $Z \in \mathbb{R}^{p \times n}$ denote the projection of $X$ onto the top $p$ eigenvectors of the graph Laplacian.  Let $g_Z(M) = tr(M^T A \frac{1}{\sqrt{2k}} B Z)$ for $ A \in \mathbb{R}^{1 \times k}, B \in \mathbb{R}^{k \times p}$.  Then as $k \to \infty$, the following are equivalent:
\begin{align*}
    \argmin_{A, B} \sum_{(i,j) \in S} (Y_{ij} - g_Z(M_{ij}))^2 \Longleftrightarrow \argmin_{w \in \mathbb{R}^{p}} \| Y_{S} - w Z\|_2^2 
\end{align*}
in the sense that $g_Z(M_{ij}) = wZ_{:, j}$.
\end{corollary*}
\end{example}

The proof follows immediately from Theorem 1 and the fact that the dual activation for $\phi(x) = x$ is $\check{\phi}(\xi) = \xi$. The example above illustrates the generality of our framework for matrix completion.  Moreover, semi-supervised learning with the graph Laplacian can naturally be extended by using the NTK for a nonlinear neural network instead of a linear neural network. Namely, instead of using the eigenvectors of the graph Laplacian, we can naturally extend the above corollary by using embeddings produced by autoencoders (Ch. 14 of \cite{goodfellow2016deep}).

\textcolor{black}{Note that the flexibility to learn a low-rank imputation or imputation with other structures via our framework is given by the feature prior, which incorporates the relationships between the coordinates of the target matrix.  Indeed, varying the feature prior can drastically change the imputation given by the NTK, and the NTK with appropriate \prior can even produce low-rank imputations, as shown by the following example below.  }

\textcolor{black}{\begin{example} Consider the Netflix problem of movie rating imputation.  Suppose the target matrix $Y$ is of the form 
\begin{align*}
    Y = \begin{bmatrix} 1 & 2 & y_{13} \\ 1 & 2 & 3 \end{bmatrix},
\end{align*}
where the rows of $Y$ represent users, the columns represent movies, and the coordinate $Y_{ij}$ represents the rating (from 1 to 5 stars) a user $i$ gave to movie $j$.  By first flattening the matrix $Y$ into $Y_v = [1, 2, y_{13}, 1, 2, 3]$, and then using our framework with feature prior, 
\begin{align*}
    Z = \begin{bmatrix} 1 & 0 & 0 & 1 & 0 & 0 \\ 0 & 1 & 0 & 0 & 1 & 0 \\ 0 & 0 & 1 & 0 & 0 & 1 \end{bmatrix} , 
\end{align*}
leads to a low rank imputed matrix
\begin{align*}
    \hat{Y} = \begin{bmatrix} 1 & 2 & 3 \\ 1 & 2 & 3\end{bmatrix}.
\end{align*}
\end{example}
The above example is simplistic in that it produces a low rank imputation by assuming that the users are identical and using a one-hot embedding for the movies. In practice, one would use a feature prior that embeds users via external metadata (e.g.~user age, gender, etc.) and our framework would predict similar ratings for users with similar metadata.}

\section{Experimental Details for Virtual Drug Screening in CMAP}
\label{appendix: C}

For this application, we consider the 978 genes $\times$ 2,130 drugs $\times$ 71 cell types ``large'' tensor from \cite{SontagDrugImputation}. From this tensor, we extract the 15,855 non-null values, and leave out the cell types (‘SNU1040’, ‘HEK293T’, ‘HS27A’), as they have less than 10 drugs in the dataset (i.e. for these cell types, we would not be able to perform 10-fold cross validation).  We exclude MCF7 from the dataset when using our method, since we use it to compute our feature prior, but we give all other methods training access to all MCF7 observations to ensure a fair comparison.  This leaves us with a dataset of 14,336 samples, which are used for imputation.   A link to download this dataset is given in \cite{SontagDrugImputation}, which we repeat here for convenience: \url{https://github.com/clinicalml/dgc_predict}. %\href{https://github.com/clinicalml/dgc_predict}{https://github.com/clinicalml/dgc\_predict}.  

For training DNPP and FaLRTC, we use the same hyper-parameters as in \cite{SontagDrugImputation}.  We implemented DNPP, mean over cell type, and our framework in Python in the above link.  We use the Matlab code from \cite{SontagDrugImputation} located via the following link: \href{https://github.com/clinicalml/dgc_predict/blob/b8bff6d757fc757aadf39034b9972db37c6da983/matlab/thirdparty/visual/FaLRTC.m}{\texttt{https://github.com/clinicalml/dgc\_predict/FaLRTC.m}}. %\href{https://github.com/clinicalml/dgc_predict/blob/b8bff6d757fc757aadf39034b9972db37c6da983/matlab/thirdparty/visual/FaLRTC.m}{https://github.com/clinicalml/dgc_predict/blob/b8bff6d757fc757aadf39034b9972db37c6da983/matlab/thirdparty/visual/link to FaLRTC.m code}.  
In order to make our results for FaLRTC accessible without Matlab, we provide the imputations from FaLRTC in the following folder:
\url{https://www.dropbox.com/sh/w23viwbm3py1dq1/AADQD3Bi_bLx4Z7X2hcLoUzXa?dl=0}.
%\href{https://www.dropbox.com/sh/w23viwbm3py1dq1/AADQD3Bi_bLx4Z7X2hcLoUzXa?dl=0}{https://www.dropbox.com/sh/w23viwbm3py1dq1/AADQD3Bi\_bLx4Z7X2hcLoUzXa?dl=0}.

% For producing the sparse subset of data in Fig.~2, we restrict our training set to contain at most 10 drug samples per cell type.  This results in a training set of 134 training examples and 14202 test samples.  

% We next describe the \priors, $X$, used for our method. 

% For all \priors, we additionally concatenated a constant (1.5) times the identity matrix to ensure that the corresponding kernel was positive definite.\footnote{We chose the constant 1.5 by tuning this parameter for Pearson r value on a single random seed.  We then used this constant for all other random seeds.} 

\section{Feature Prior for Drug Response Imputation}
\label{appendix: D}
DNPP performs well for imputing the effect of drugs on cell types that have many observations in the training set, but performs poorly when imputing the effect of drugs on cell types with few observations in the training set.  Thus, to improve on DNPP, we use a dual \prior: one for imputing the effect of drugs on cell types with many (at least 150) observations in the training set (the dense regime), and another for imputing the effect of drugs on cell types with few (at most 150) observations in the training set (the sparse regime).  

Since DNPP and FaLRTC both yield an imputation that captures similarity between cell type and drug combinations in the large observation regime, we can use the output of one of these methods as the \prior for those cell types that had greater than 150 drugs in the training set.  In particular, we chose the output of FaLRTC for the \prior in the dense regime since applying our method with this \prior yielded superior results. For all observed examples that were in the training set, we use the gene expression for the observation itself as the encoding.   For all \priors, we additionally concatenated a constant (1.5) times the identity matrix to ensure that the corresponding kernel is positive definite\footnote{We chose the constant 1.5 by tuning this parameter to give highest Pearson r value on seed 512.  We then used this constant for all other random seeds.}.  We then solved kernel regression exactly (using the numpy solve function \cite{numpy1}) for the NTK of a 1-hidden layer ReLU network.  

For those cell types with few (less than 150 observations) in the training set, we used a \prior that concatenates an embedding of the cell type and an embedding of the drug type. For the drug embedding, we used the gene expression of MCF7 treated with the same drug as the drug embedding, if available in the training set.  If this vector was not available in the training set, we simply used the mean of all MCF7 observations.  For the cell type embedding, we used the mean of all observations for the corresponding cell type available in the training set.  We then normalized each cell embedding to have the same norm as the drug embedding to balance their contributions to dot products computed for the kernel.  We re-scaled the embedding for the cell type by a factor of 1.25 to give the cell type additional weight over drug type\footnote{This hyperparameter was selected to maximize Pearson r value for seed 512 and then fixed across all other random seeds.}.  Lastly, we normalized the concatenation of the embeddings and solved kernel regression via the closed form in Theorem 1.  We refer to this \prior as the \textit{MCF7 reference prior}.  

The code for computing our \priors is available at \url{https://github.com/uhlerlab/ntk_matrix_completion}. %\href{https://github.com/uhlerlab/ntk_matrix_completion}{https://github.com/uhlerlab/ntk\_matrix\_completion}.  

\section{One-hot Encoding for Drugs is Equivalent to Imputation with Mean Over Cell Type}
\label{appendix: E}

The following result shows that using a \prior consisting of a one-hot embedding for drugs leads to performing imputation using the mean over all observations for a given cell type.

\begin{prop}
Let $Y \in \mathbb{R}^{m \times n}$ denote the gene expression vectors for cell type $c$ with drugs $\{d_j\}_{j=1}^{n}$, such that columns $\{y^{(j)}\}_{i=1}^{\ell}$ are observed and columns $\{y^{(j)}\}_{i=\ell+1}^{n}$ are missing.  Let $A \in \mathbb{R}^{m \times k}, B \in \mathbb{R}^{k \times p}$, $\phi(x) = \max(x, 0)$,  $g: \mathbb{R}^{m \times n} \to \mathbb{R}$ such that: 
\begin{align*}
    g(M) = tr\left(M^T A \frac{\sqrt{2}}{\sqrt{k \ell}} \phi(BZ) \right),
\end{align*}
where $Z = I_{n \times n}$ (i.e. a one-hot encoding of the drug).  Then for $i \in [m], j > \ell$,  the solution to kernel ridge-less regression with the NTK for $g$ is: 
\begin{align*}
\tilde{g}(M_{ij}) = \left(\frac{1}{2\pi - 1} - \frac{\ell}{(2\pi -1)(2\pi - 1 + \pi \ell) }\right) \left(\frac{1}{\ell}\sum_{j=1}^{\ell} y_i^{(j)}\right). 
\end{align*}
\end{prop}

\begin{proof}  
The proof relies on the fact that the kernel matrix $K$ for $g$ is a block diagonal matrix.  In particular, as shown in the example in Section 3, there is one block, $K_{B_i} \in \mathbb{R}^{\ell \times \ell}$, for each row of $Y$ (i.e. $m$ blocks), and $K_{B_i}$ has diagonal entries $\kappa(1) = 2$ and off-diagonal entries $\kappa(0) = \frac{1}{\pi}$.  Hence, each block of the kernel matrix can be written as: 
\begin{align*}
    K_{B_i} = \frac{1}{\ell} \left(\left(2 - \frac{1}{\pi}\right) I_{\ell \times \ell} + \frac{1}{\pi} J \right),
\end{align*}
where $J \in \mathbb{R}^{\ell \times \ell}$ is the all ones matrix.   By the Sherman-Morrison formula, 
\begin{align*}
    K_{B_i}^{-1} = \frac{1}{\ell}\left(\frac{\pi}{2\pi -1} I - \frac{\pi}{(2\pi - 1)(2\pi - 1 + \pi \ell)} J \right),
\end{align*}
and thus 
\begin{align*}
    \tilde{g}(M_{ij}) = \begin{bmatrix}y_i^{(1)} & y_i^{(2)} & \ldots & y_i^{(\ell)} \end{bmatrix} K_{B_i}^{-1} \mathbf{1}  \frac{1}{\pi \ell},
\end{align*}
where $\mathbf{1} \in \mathbb{R}^{\ell}$ is the all ones vector.  Hence, 
\begin{align*}
    \tilde{g}(M_{ij}) = \left(\frac{1}{2\pi - 1} - \frac{\ell}{(2\pi -1)(2\pi - 1 + \pi \ell) }\right) \left( \frac{1}{\ell} \sum_{j=1}^{\ell} y_{i}^{(j)} \right),
\end{align*}
which completes the proof.
\end{proof}

\section{\PriorAllCaps Corresponding to Previous Algorithms} 
\label{appendix: F}
As discussed in Section 2 of the main text, our framework provides a direct approach for improving upon previous methods for virtual drug screening.  Using the output of DNPP and FaLRTC as the \prior in our framework leads to an improvement; namely, across every round and fold in 5 rounds of 10-fold cross validation (using seeds $149, 10, 53, 77, 1928$), we find that our method with the DNPP output as a \prior outperforms DNPP and that our method with the FaLRTC output as a \prior outperforms FaLRTC.  This is demonstrated in Figs.~\ref{fig: SI DNPP vs. DNPP Prior} and \ref{fig: SI FaLRTC vs. FaLRTC Prior}.  

\section{Performance of Methods on Sparse versus Dense Subsets}
\label{appendix: G}
We demonstrate in Fig.~\ref{fig: SI Dense vs. Sparse Regime} that DNPP is effective for imputation on the dense regime (i.e.~for those drug/cell type pairs with over 150 profiles), but  not as effective in the sparse regime (i.e.~for those drug/cell type pairs with less than 150 profiles).  FaLRTC seems to perform comparably between the dense and the sparse regime, but under-performs DNPP on the full dataset.

\section{Metrics for Evaluation in Drug Response Imputation}
\label{appendix: H}
Let $\hat{Y} \in \mathbb{R}^{m \times n}$ denote the concatenatation of the test predictions for all 10 folds and let $Y^* \in \mathbb{R}^{m \times n}$ denote the ground truth.  We use ${y^*}^{(i)}$  to denote the $i^{th}$ column of $Y^*$.  Let $\bar{y}^{(i)} = c_i \mathbf{1}$ where $c_i = \sum_{j=1}^{m} y_j^{(i)}$.   For $A \in \mathbb{R}^{a \times b}$, let $A_v \in \mathbb{R}^{a \cdot b}$ denote the vectorized version of $A$.  We use the following 3 metrics for evaluating the effectiveness of a given imputation method.   All evaluation metrics have a maximum value of $1$. 
\vspace{2mm}

1. \textit{Pearson r value}: This evaluation metric was used in \cite{SontagDrugImputation} and is given by:  
\begin{align*}
    v = \frac{\langle \hat{Y}_v, Y_v^* \rangle}{\|\hat{Y}_v\|_2 \|Y_v^*\|_2}.
\end{align*}

\vspace{2mm}

2.\textit{Mean $R^2$}: This evaluation metric is given by: 
\begin{align*}
    v = \frac{1}{n} \sum_{i=1}^{n} \left( 1 - \frac{\sum_{j=1}^{m}( \hat{y}_j^{(i)} - {y_j^*}^{(i)})^2} {\sum_{j=1}^{m} ({y_j^*}^{(i)} - \bar{y}_j^{(i)} )^2 }\right).
\end{align*}
\vspace{2mm}

3. \textit{Mean Cosine Similarity}: This evaluation metric is given by: 
\begin{align*}
    v = \frac{1}{n} \sum_{i=1}^{n} \frac{\langle \hat{y}^{(i)}, {y^*}^{(i)} \rangle }{\|\hat{y}^{(i)}\|_2 \|{y^*}^{(i)}\|_2}.
\end{align*}

\section{Statistical Significance of NTK on Drug Response Imputation} 
\label{appendix: I}
In experiments on the full dataset, we use 10-fold cross validation and 5 random seeds (149, 10, 77, 53, 1928) for comparing our method to DNPP from \cite{SontagDrugImputation}. For each fold, we ensure that $10\%$ of the drugs for each cell type are present in the test set.  To determine the statistical significance of  our method for improving over DNPP, we use a one-sided test with the following corrected repeated k-fold cv test statistic for $r$ rounds of $k$-fold cross validation (as described in Section 3.3 of \cite{CrossValidationTTest}): 
\begin{align*}
    t = \frac{\frac{1}{kr} \sum_{i=1}^{k} \sum_{j=1}^{r} d_{ij} }{\left( \frac{1}{kr} + \frac{n_2}{n_1}\right)\hat{\sigma}^2},
\end{align*}
where $d_{ij}$ is the difference between the evaluation metric for our method (the output of FaLRTC as the \prior for the dense regime and the MCF7 reference \prior for the sparse regime) and that of the DNPP for fold $k$ of round $j$, $\hat{\sigma}$ is the estimated variance of the differences $d_{ij}$, and $n_1$ is the number of samples used for training and $n_2$ is the number of samples used for testing (i.e. $\frac{n_2}{n1} \approx \frac{1}{9}$ for our setting).  This statistic is distributed according to a t-distribution with $kr - 1$ degrees of freedom.  For the mean $R^2$, we obtain $t = 18.29$ and a corresponding p-value of $7.7 \cdot 10^{-24}$.  For the mean cosine similarity, we obtain $t = 14.75$ and a p-value of $5.9 \cdot 10^{-20}$.  Thus, at a significance level of $.01$, we reject the null hypothesis that our method and DNPP have the same performance.

\section{Matrix Completion with the CNTK}
\label{appendix: J}
We repeat Proposition 1 from the main text and present the proof below. The tensor $\Theta \in \mathbb{R}^{m \times n \times m \times n}$ was defined and used in the computation of the CNTK for classification in \cite{CNTKArora}.  

\begin{prop*}
\label{prop: CNTK Matrix Completion}
Let $f_Z(\mathbf{W})$ be a $d$ layer convolutional network used to map from the  \prior $Z \in \mathbb{R}^{c \times r \times s}$ to the target matrix $Y \in \mathbb{R}^{m \times n}$. Then as the number of convolutional filters per layer tends to infinity, the CNTK of $f_Z(\mathbf{W})$ is given by:
\begin{align}
    K(M_{ij}, M_{i'j'}) = [\Theta^{(d)}(Z, Z)]_{i,j,i',j'},
\end{align}
where $M_{ij}, M_{i'j'} \in \mathbb{R}^{m \times n}$ denote indicator matrices. 
\end{prop*}

\begin{proof}
The proof follows almost immediately from the derivation of the CNTK for classification provided in \cite{CNTKArora}.  Namely, let $g(M) = M^T f_Z(\mathbf{W})$ for $M \in \mathbb{R}^{m \times n}$. Then, we have that: 
\begin{align*}
    \frac{\partial g(M)}{\partial \mathbf{W}_{\alpha, \beta}} = \sum_{i=1}^{m} \sum_{j=1}^{n} M_{i,j} \frac{\partial f_Z(\mathbf{W})_{i,j}}{\partial \mathbf{W}_{\alpha, \beta}}.
\end{align*}
Thus, the kernel at the indicator matrices $M_{ij}, M_{i'j'}$ is given by: 
\begin{align*}
    K(M_{ij}, M_{i'j'}) = \frac{\partial g(M_{ij})}{\partial \mathbf{W}_{\alpha, \beta}} \frac{\partial g(M_{i'j'})}{\partial \mathbf{W}_{\alpha, \beta}} = \frac{\partial f_Z(\mathbf{W})_{i,j}}{\partial \mathbf{W}_{\alpha, \beta}} \frac{\partial f_Z(\mathbf{W})_{i', j'}}{\partial \mathbf{W}_{\alpha, \beta}} = [\Theta(Z, Z)]_{i,j, i', j'},
\end{align*}
which completes the proof.
\end{proof}
% \begin{theorem*}[$L-1$ Hidden Layer, ReLU]
% Let $Z \in \mathbb{R}^{c \times r \times s}$ denote an input image.  Let $g(Z)$ denote a convolutional neural network: 
% \begin{align*}
%     Z^{(0)} &= Z \\
%     Z^{(i)} &= \frac{\sqrt{2}}{q_{i} \sqrt{k}}  \phi(W^{i-1} * Z^{(i-1)}) ~~ \text{for $i \in [L-1]$} \\
%     g(Z) &= W^{(L)} * Z^{(L-1)}
% \end{align*}
% where $W^{(i)}$ has $k$ filters of size $q_i$ and circular padding.  If $f_Z(M) = tr (M^T g(Z))$, then the CNTK $K$ for $f_Z$ is: 
% \begin{align*}
%     K(M_{ij}, M_{i'j'}) = [\Theta^{(L)}(X, X)]_{ij, i'j'} ~;
% \end{align*}
% where $\Theta^{(L)} \in \mathbb{R}^{d \times d \times d \times d}$ is the tensor used to compute the CNTK for classification (Section 4 of \cite{CNTKArora}).
% \end{theorem*}

Below we additionally present an explicit derivation for the 1 hidden layer case for ReLU networks.  This derivation will be useful in understanding the connection between the CNTK for matrix completion with semi-supervised learning from coordinate embeddings (i.e. Theorem 2 of the main text).  

\begin{prop*}[1 Hidden Layer Convolutional Network]
Let $Z \in \mathbb{R}^{c \times m \times n}$ denote the \prior.  Let $*$ denote the neural network convolution operator and let $f_Z(\mathbf{W}) = A * \frac{\sqrt{c}}{q \sqrt{k}}\phi(B * Z)  $ denote a 1 hidden layer convolutional network where $B$ has $k$ filters of size $q \times q \times c$ with circular padding, $A$ has $1$ filter of size $q \times q \times k$ with circular padding for odd $q$, $\phi$ is a homogeneous activation function of degree 1, and $c^2 = \frac{1}{\mathbb{E}_{u \sim \mathcal{N}(0, 1)} [\phi(u)^2]} $.  Let $K^{(0)}, \tilde{K}^{(0)}, \Sigma^{(0)} \in \mathbb{R}^{m \times n \times m \times n}$ such that: 
\begin{align*}
    \Sigma^{(0)}(i, j, i', j') = K^{(0)}(i, j, i', j') &= \sum_{\ell=1}^{c} \sum_{ -\frac{q+1}{2} \leq m, n \leq \frac{q+1}{2} } Z_{\ell, i+m, j+n} Z_{\ell, i' + m, j' + n}.
\end{align*}
If $M_{ij}$ and $M_{i'j'}$ are indicator matrices, then as $k \to \infty$, the CNTK for $f_Z(\mathbf{W})$ is given by: 
\begin{align*}
    K(M_{ij}, M_{i'j'}) &= \frac{1}{q^2}\sum_{-\frac{q+1}{2} \leq a, b \leq \frac{q+1}{2}} \Sigma^{(1)}(i + a, j + b, i' + a, j' + b)  \\
    &~~~~~~~~~~+ \dot{\Sigma}^{(1)}(i + a, j + b, i' + a, j' + b) K^{(0)}(i + a, j + b, i' + a, j' + b),
\end{align*}
where 
\begin{align*}
    \Sigma^{(1)}(i, j, i', j') &= \sqrt{\Sigma^{(0)}(i, j, i, j) \Sigma^{(0)}(i', j', i', j')} \hat{\phi}\left( \frac{\Sigma^{(0)}(i, j, i', j')}{\sqrt{\Sigma^{(0)}(i, j, i, j) \Sigma^{(0)}(i', j', i', j')}} \right), \\
     \dot{\Sigma}^{(1)}(i, j, i', j') &= \frac{d\hat{\phi}}{d\xi}\left( \frac{\Sigma^{(0)}(i, j, i', j')}{\sqrt{\Sigma^{(0)}(i, j, i, j) \Sigma^{(0)}(i', j', i', j')}} \right).
\end{align*}
\end{prop*}

\begin{proof}
We provide the proof for the case of $1$ input channel ($c = 1$) below.  The proof follows analogously for the case of multiple input channels.  Let $g(M) = tr(M^T f_Z(\mathbf{W}))$.  Let $Y^{(\ell)}$ denote channel $\ell$ of $\frac{2}{\sqrt{k}}\phi (B * Z)$ and let $H = A * \frac{\sqrt{2}}{q\sqrt{k}}\phi (B* Z)$.  We thus have that 
\begin{align*}
    Y_{ij}^{(\ell)} &= \frac{\sqrt{c}}{q\sqrt{k}} \phi\left(  \sum_{-\frac{q+1}{2} \leq a, b \leq \frac{q+1}{2}} Z_{i + a, j+b} B_{a, b} ^{(\ell)}\right), \\
    H_{ij} &= \sum_{\ell=1}^{k} \sum_{-\frac{q+1}{2} \leq a, b \leq \frac{q+1}{2}} Y_{i + a, j+ b}^{(\ell)} A_{a, b}^{(\ell)}, \\
    g(M) &=  \sum_{1 \leq i, j \leq d} M_{ij} H_{ij}.
\end{align*}
Now we compute the partial derivatives of $f$ with respect to the parameters $A_{a,b}^{(\ell)}$ and $B_{m, n}^{(\ell)}$: 
\begin{align*}
    \frac{\partial g(M_{ij})}{\partial A_{a,b}^{(\ell)}} &=  Y_{i + a, j+ b}^{(\ell)},  \\
    \frac{\partial g(M_{ij})}{\partial B_{m, n}^{(\ell)}} &=   \sum_{-\frac{q+1}{2} \leq a, b \leq \frac{q+1}{2}}  A_{a, b}^{(\ell)} \frac{\sqrt{c}}{q\sqrt{k}} \phi'\left(  \sum_{-\frac{q+1}{2} \leq a', b' \leq \frac{q+1}{2}} Z_{i + a + a', j+b + b'} B_{a', b'} ^{(\ell)} \right)  Z_{i + a + m, j + b + n}.
\end{align*}

As $k \to \infty$, the CNTK converges in probability to:
\begin{align}
    \label{eq: CNTK 1 hidden layer ReLU}
        K(M_{ij}, M_{i'j'}) = \mathbb{E}_{A_{a, b}^{(\ell)}, B_{m,n}^{(\ell)} \sim \mathcal{N}(0, 1)}\left[ \sum_{\ell=1}^{k} \sum_{a, b} \frac{\partial g(M_{ij})}{\partial A_{a, b}^{(\ell)}}\frac{\partial g(M_{i'j'})}{\partial A_{a, b}^{(\ell)}} + \sum_{\ell=1}^{k} \sum_{m, n} \frac{\partial g(M_{ij})}{\partial B_{m, n}^{(\ell)}} \frac{\partial g(M_{i'j'})}{\partial B_{m, n}^{(\ell)}} \right].
\end{align}
This expression can be simplified as follows: 
\begin{align*}
    K(M_{ij}, M_{i'j'}) &=  \sum_{-\frac{q+1}{2} \leq a, b \leq \frac{q+1}{2}} \Sigma^{(1)}(i + a, j + b, i' + a, j' + b) \\
    &~~~~~~~~~~ + \dot{\Sigma}^{(1)}(i + a, j + b, i' + a, j' + b) K^{(0)}(i + a, j + b, i' + a, j' + b),
\end{align*} 
where we have: 
\begin{align*}
    \Sigma^{(1)} &= \frac{c}{q^2}\mathbb{E}_{B_{a', b'}^{(\ell)}} \left[  \sum_{a, b} \phi\left(  \sum_{-\frac{q+1}{2} \leq a', b'\leq \frac{q+1}{2}} Z_{i + a + a', j+b + b'} B_{a', b'} ^{(\ell)}\right) \phi\left(  \sum_{-\frac{q+1}{2} \leq a', b'\leq \frac{q+1}{2}} Z_{i' + a + a', j' +b + b'} B_{a', b'} ^{(\ell)}\right) \right]\\
    \dot{\Sigma}^{(1)} &= \frac{c}{q^2}\mathbb{E}_{B_{a', b'}^{(\ell)}}  \left[\sum_{a,b} \phi'\left(  \sum_{-\frac{q+1}{2} \leq a', b' \leq \frac{q+1}{2}} Z_{i + a + a', j+b + b'} B_{a', b'} ^{(\ell)} \right)  \phi'\left(  \sum_{-\frac{q+1}{2} \leq a', b' \leq \frac{q+1}{2}} Z_{i' + a + a', j'+b + b'} B_{a', b'} ^{(\ell)} \right) \right]    
\end{align*}
Lastly, we reduce the above expressions by substituting in the values for $\Sigma^{(0)}$ from the statement of the proposition.  Namely, let
\begin{align*}
    u &= \sum_{a, b} \phi\left(  \sum_{-\frac{q+1}{2} \leq a', b'\leq \frac{q+1}{2}} Z_{i + a + a', j+b + b'} B_{a', b'} ^{(\ell)}\right),  \\
    v &=  \sum_{a, b} \phi\left(  \sum_{-\frac{q+1}{2} \leq a', b'\leq \frac{q+1}{2}} Z_{i' + a + a', j' +b + b'} B_{a', b'} ^{(\ell)}\right).
\end{align*}
Then, the above expressions for $\Sigma^{(1)}, \dot{\Sigma^{(1)}}$ simplify to: 
\begin{align*}
    \Sigma^{(1)} &= \frac{c}{q^2}\mathbb{E}_{B_{a', b'}^{(\ell)}} \left[ \phi(u) \phi(v) \right], \\
    \dot{\Sigma}^{(1)} &= \frac{c}{q^2}\mathbb{E}_{B_{a', b'}^{(\ell)}} \left[ \phi'(u) \phi'(v) \right].    
\end{align*}
Hence, we can use the formula for the dual activation of the ReLU to conclude that: 
\begin{align*}
    \Sigma^{(1)}(i, j, i', j') &= \frac{1}{q^2}\sqrt{\Sigma^{(0)}(i, j, i, j) \Sigma^{(0)}(i', j', i', j')} \check{\phi}\left( \frac{\Sigma^{(0)}(i, j, i', j')}{\sqrt{\Sigma^{(0)}(i, j, i, j) \Sigma^{(0)}(i', j', i', j')}} \right), \\
     \dot{\Sigma}^{(1)}(i, j, i', j') &= \frac{1}{q^2}\frac{d\check{\phi}}{d\xi}\left( \frac{\Sigma^{(0)}(i, j, i', j')}{\sqrt{\Sigma^{(0)}(i, j, i, j) \Sigma^{(0)}(i', j', i', j')}} \right).
\end{align*}
Lastly, we complete the proof by substituting these expressions for $\Sigma^{(1)}, \dot{\Sigma}^{(1)}$ into the expression for $K(M_{ij}, M_{i'j'})$ above.
\end{proof}

\textcolor{black}{As implied by Proposition 1 above, the CNTK is a functional of \emph{pairs of coordinates} of images, while the usual CNTK for classification operates on pairs of images~\cite{CNTKArora}. To be more specific, consider the setting where the target matrix $Y$ is in $\mathbb{R}^{m \times n}$. Then, the CNTK for matrix completion that we compute lies in $\mathbb{R}^{mn \times mn}$.  On the other hand, when given $n$ images for classification, the CNTK computed in \cite{CNTKArora} lies in $\mathbb{R}^{n \times n}$ and does not depend on the image size.}

\section{Equivalence with Semi-Supervised Learning for the CNTK}
\label{appendix: K}
In the following, we present the statement and proof of Theorem 2 from the main text with the precise form for $\tilde{\psi}$.   

\begin{theorem*}
Consider a convolutional network, $f_Z(\mathbf{W})$, with $d$ hidden layers with homogeneous activation of degree 1 and in which all filters have size $q$ and circular padding.  Let $Z \in \mathbb{R}^{c \times m \times n}$ satisfy:
\begin{align*}
\sum_{\ell=1}^{c} \sum_{ -\alpha \leq a, b \leq \alpha } Z_{\ell, i+a, j+b} Z_{\ell, i' + a, j' + b} = \psi ( | i - i'|, | j - j'|)
\end{align*}
for some $\psi: \mathbb{R}^2 \to \mathbb{R}$ with maximum at $(0,0)$ and $\alpha = \frac{q-1}{2}$ (odd $q$). Then as the number of convolutional filters per layer goes to infinity, the CNTK is given by: 
\begin{align*}
    K_d(M_{ij}, M_{i'j'}) &= \tilde{\psi}(|i-i'|, |j-j'|) \\
    &= \check{\phi}^{(d)}\left(\frac{\psi ( | i - i'|, | j - j'|)}{\psi(0,0)}  \right) \psi(0,0) + K_{d-1}(M_{ij}, M_{i'j'}) \frac{d \check{\phi}}{d \xi} \left( \check{\phi}^{(d-1)} \left( \frac{\psi ( | i - i'|, | j - j'|)}{\psi(0,0)} \right) \right),
\end{align*}
where $\check{\phi}$ is the dual activation of $\phi$, $\check{\phi}^{(d)}(\xi) = \check{\phi} ( \check{\phi}^{(d-1)}(\xi))$ with $\check{\phi}^{(0)}(\xi)  = \xi$, and $K_{0}(M_{ij}, M_{i'j'}) = \psi(|i - i'|, |j - j'|)$. 
\end{theorem*}
% \end{proof}

% \begin{prop*}
% Let $X \in \mathbb{R}^{c \times m \times n}$  and let $f: \mathbb{R}^{c \times m \times n} \to \mathbb{R}^{c \times m \times n}$ be a 1 hidden layer convolutional network such that $f(M) = tr(M^T A * \frac{\sqrt{2}}{q \sqrt{k}}\phi(B * X))$ where $B$ has $k$ filters of size $q \times q \times c$ with circular padding, $A$ has $1$ filter of size $q \times q \times k$ with circular padding for odd $q$, and $\phi$ is the ReLU activation function. If $X$ satisfies:  
% \begin{align*}
% \sum_{\ell=1}^{c} \sum_{-\frac{q-1}{2} \leq a, b \leq \frac{q-1}{2}} X_{\ell, i+a, j+b} X_{\ell, i' + a, j' + b} = \psi ( | i - i'|, | j - j'|) ~ ; 
% \end{align*}
% for $\psi: \mathbb{R}^2 \to \mathbb{R}$ with maximum at $(0,0)$, then the CNTK simplifies to:
% \begin{align*}
%     K(M_{ij}, M_{i'j'}) = C_2 \hat{\phi}\left(\frac{C_1}{C_2} \right) + \frac{d\hat{\phi}}{d\xi}\left(\frac{C_1}{C_2} \right) C_1 
% \end{align*}
% where $C_1  = \psi(|i - i'|, |j-j'|)$ and $C_2 = \psi(0,0)$.
% \end{prop*}

% \begin{proof}
\begin{proof}
We prove this by induction on the number of hidden layers $d$.  We begin with the base case for $d=1$: The proof for this case follows from the proof of the Proposition for 1 hidden convolutional networks in Appendix \ref{appendix: J}.  Namely, we have: 
\begin{align*}
      K(M_{ij}, M_{i'j'}) &= \sum_{-\frac{q-1}{2} \leq a, b \leq \frac{q-1}{2}} \Sigma^{(1)}(i + a, j + b, i' + a, j' + b)  \\
    &~~~~~~~~~~+ \dot{\Sigma}^{(1)}(i + a, j + b, i' + a, j' + b) K^{(0)}(i + a, j + b, i' + a, j' + b),
\end{align*}
where 
\begin{align*}
    \Sigma^{(1)}(i, j, i', j') &= \frac{1}{q^2} \sqrt{\Sigma^{(0)}(i, j, i, j) \Sigma^{(0)}(i', j', i', j')} \check{\phi}\left( \frac{\Sigma^{(0)}(i, j, i', j')}{\sqrt{\Sigma^{(0)}(i, j, i, j) \Sigma^{(0)}(i', j', i', j')}} \right), \\
     \dot{\Sigma}^{(1)}(i, j, i', j') &= \frac{1}{q^2} \frac{d\check{\phi}}{d\xi}\left( \frac{\Sigma^{(0)}(i, j, i', j')}{\sqrt{\Sigma^{(0)}(i, j, i, j) \Sigma^{(0)}(i', j', i', j')}} \right).
\end{align*}
Now since $\Sigma^{(0)}(i, j, i', j') = \psi(|i - i'|, |j-j'|)$, we conclude that 
\begin{align*}
    \Sigma^{(1)}(i, j, i', j') &= \frac{1}{q^2}  \psi(0, 0) \check{\phi}\left(\frac{\psi(|i - i'|, |j - j'|)}{\psi(0,0)}\right), \\
    \dot{\Sigma}^{(1)} &= \frac{1}{q^2}  \frac{d\check{\phi}}{d\xi} \left(\frac{\psi(|i - i'|, |j - j'|)}{\psi(0,0)} \right).
\end{align*}
Substituting the above into the expression for $K(M_{ij}, M_{i'j'})$, we obtain
\begin{align*}
    K(M_{ij}, M_{i'j'}) &= \frac{1}{q^2} \sum_{-\frac{q+1}{2} \leq a, b \leq \frac{q+1}{2}} \psi(0, 0) \check{\phi}\left(\frac{\psi(|i - i'|, |j - j'|)}{\psi(0,0)}\right) \\
    &~~~~~~~~~~ + \psi(|i - i'|, |j - j'|) \frac{d\check{\phi}}{d\xi} \left(\frac{\psi(|i - i'|, |j - j'|)}{\psi(0,0)} \right)
\end{align*}
Note that the summand no longer depends on $a, b$, and thus we conclude that
\begin{align*}
    K(M_{ij}, M_{i'j'}) = \psi(0,0) \check{\phi}\left(\frac{\psi(|i - i'|, |j-j'|)}{\psi(0,0)} \right) + \frac{d\check{\phi}}{d\xi}\left(\frac{\psi(|i - i'|, |j-j'|)}{\psi(0,0)} \right) \psi(|i - i'|, |j-j'|), 
\end{align*}
which completes the base case.

For the inductive step, we assume that the following holds for depth $d-1$: 
\begin{align*}
    \Sigma^{(d-1)}(M_{ij}, M_{i'j'}) &= \frac{1}{q^2} \check{\phi}^{(d-1)}\left(\frac{\psi ( | i - i'|, | j - j'|)}{\psi(0,0)}  \right)  \psi(0,0),\\
    \dot{\Sigma}^{(d-1)}(M_{ij}, M_{i'j'}) &=  \frac{1}{q^2} \frac{d \check{\phi}}{d \xi} \left( \check{\phi}^{(d-2)} \left( \frac{\psi ( | i - i'|, | j - j'|)}{\psi(0,0)} \right) \right),\\
    K_{d-1}(M_{ij}, M_{i'j'}) &= q^2\Sigma^{(d-1)}(M_{ij}, M_{i'j'}) + q^2 K_{d-2}(M_{ij}, M_{i'j'}) \dot{\Sigma}^{(d-1)}(M_{ij}, M_{i'j'}),
\end{align*}
and assume that $K_{d-1}(M_{ij}, M_{i'j'}) = K_{d-1}(M_{i+a, j+b}, M_{i'+a, j'+b})$ for any $a, b \in \mathbb{Z}$ satisfying $i+a, i'+a \in [m]$ and $j+b, j'+b \in [n]$ (i.e. assume that $K_{d-1}$ is \textit{shift invariant}).  Now, let $S^{(d-1)}(M_{ij}, M_{i'j'})$ be defined as follows: 
\begin{align*}
    S^{(d-1)}(M_{ij}, M_{i'j'}) = \sum_{-\frac{q-1}{2} \leq a, b, \leq \frac{q-1}{2}} \Sigma^{(d-1)}(M_{i+a, j+b}, M_{i'+a, j'+b}) = \check{\phi}^{(d-1)}\left(\frac{\psi ( | i - i'|, | j - j'|)}{\psi(0,0)}  \right)  \psi(0,0).
\end{align*}
Then, by the derivation of the CNTK in \cite{CNTKArora}, we obtain
\begin{align*}
    \Sigma^{(d)}(M_{ij}, M_{i'j'}) &= \frac{1}{q^2} \check{\phi} \left( \frac{S^{(d-1)}(M_{ij}, M_{i'j'})}{\sqrt{S^{(d-1)}(M_{ij}, M_{ij}) S^{(d-1)}(M_{i'j'}, M_{i'j'})} } \right) \sqrt{S^{(d-1)}(M_{ij}, M_{ij}) S^{(d-1)}(M_{i'j'}, M_{i'j'})}\\
    &=  \frac{1}{q^2} \check{\phi} \left(\check{\phi}^{(d-1)}\left(\frac{\psi ( | i - i'|, | j - j'|)}{\psi(0,0)}  \right) \right) \psi(0, 0),
\end{align*}
where the last equality follows from the fact that $\check{\phi}(1) = 1$.  Following an analogous derivation for $\dot{\Sigma}^{(d-1)}$, we obtain that 
\begin{align*}
      \dot{\Sigma}^{(d)}(M_{ij}, M_{i'j'}) &=\frac{1}{q^2} \frac{d \check{\phi}}{d \xi} \left(\check{\phi}^{(d-1)}\left(\frac{\psi ( | i - i'|, | j - j'|)}{\psi(0,0)}  \right) \right). 
\end{align*}
Hence, the CNTK $K_{d}(M_{ij}, M_{i'j'})$ is given by: 
\begin{align*}
    K_d(M_{ij}, M_{i'j'}) &=  \sum_{-\frac{q-1}{2} \leq a, b, \leq \frac{q-1}{2}} \Sigma^{(d)}(M_{i+a, j+b}, M_{i'+a, j'+b}) \\
    &~~~~~~~~~~ +  K_{d-1}(M_{i+a,j+b}, M_{i'+a,j'+b}) \dot{\Sigma}^{(d)}(M_{i+a,j+b}, M_{i'+a,j'+b}) \\
    &= q^2\Sigma^{(d)}(M_{ij}, M_{i'j'}) + q^2 K_{d-1}(M_{ij}, M_{i'j'}) \dot{\Sigma}^{(d)}(M_{ij}, M_{i'j'}),
\end{align*}
where the last line follows from the shift invariance of $K_{d-1}$.  Lastly, we have that $K_d$ is shift invariant since all of the terms $\Sigma^{(d)}, \dot{\Sigma}^{(d)}$ and $K_{d-1}$ are shift invariant.  Hence, the induction is complete and the theorem follows.  
\end{proof}

% \sum_{-\frac{q-1}{2} \leq a, b, \leq \frac{q-1}{2}} \Sigma^{(d-1)}(M_{i+a, j+b}, M_{i'+a, j'+b}) +  K_{d-2}(M_{i+a,j+b}, M_{i'+a,j'+b}) \dot{\Sigma}^{(d-1)}(M_{i+a,j+b}, M_{i'+a,j'+b}) \\
% \begin{corollary}
% \label{corollary: semi-supervised learning CNTK}
% In the setting of Proposition \ref{prop: Coordinate only CNTK}, let $\psi(|i - i'|, |j - j'|) = \exp \left(-L \|(i,j) - (i',j')\|_2^2 \right)$ (i.e. the Gaussian kernel) and let $\zeta$ be the feature map such that $\psi(|i - i'|, |j - j'|) = \langle \zeta (i, j), \zeta (i', j') \rangle_{\ell^{2}}$. Let $\zeta_p$ denote the projection onto the first $p$ coordinates of $\zeta$.  Let $A \in \mathbb{R}^{1 \times k}, B \in \mathbb{R}^{k \times p}$, and $\phi: \mathbb{R} \to \mathbb{R}$  be the ReLU nonlinearity.  Then, as $k \to \infty$ and $p \to \infty$, solving kernel regression with the CNTK for image inpainting is equivalent to semi-supervised learning with the following fully connected neural network:
% \begin{align*}
%     f(M_{ij}) = A \frac{\sqrt{2}}{\sqrt{k}}\phi(B \zeta_p((i, j))) ~~ ; 
% \end{align*}
% in the sense that the solutions produced by both methods are equivalent.  
% \end{corollary}

\section{Derivation of the CNTK for Matrix Completion with Modern Architectures}
\label{appendix: L}
Below, we derive the CNTK for networks with fixed linear transformations.  We note a similar formula appears in the Appendix of \cite{DenoisingNTK}, but does not appear to be derived for the cases of nearest neighbor upsampling, nearest neighbor downsampling, and bilinear upsampling. 

% We will make use of the recursive formula for the CNTK from \cite{CNTKArora} in deriving the formulas below, and in particular, we will analyze the impact of each of these layers on the terms $\Sigma, \dot{\Sigma}, K$ used in the recursions in SI Appendices H and I. 

\begin{prop*}
Let $g(M) = tr(M^T A f_Z(\mathbf{W}))$ denote a neural network where $A \in \mathbb{R}^{mn \times pq}$ is a fixed (i.e.~non-trainable) linear transformation and $f_Z(\mathbf{W})$ is a convolutional network under the NTK parameterization\footnote{We assume $A$ operates on the vectorized version of $f_Z(\mathbf{W})$ and then the output is reshaped to size $m \times n$ before multiplication by $M^T$.}.  Then the CNTK, $K \in \mathbb{R}^{mn \times mn}$, for $g$  is given by: 
\begin{align*}
    K & = A K_f A^T  \implies
    K(M_{ij}, M_{i'j'}) =  \sum_{a=1}^{pq} \sum_{b=1}^{pq} A_{v(i,j), a} A_{v(i',j'), b} K_f(M_{v_1^{-1}(a), v_2^{-1}(a)}, M_{v_1^{-1}(b), v_2^{-1}(b)}), 
\end{align*}
where $v: \mathbb{R}^{2} \to \mathbb{R}$ is the bijective map from a coordinate $(i,j)$ in a matrix $B$ to its position in the vectorized version of $B$ and $K_f \in \mathbb{R}^{pq \times pq}$ is the CNTK for $f$.
\end{prop*}

\begin{proof}
Let  $w_{p}$ denote a weight in $f$ and let $\textbf{w}$ denote the vector of all weights in $f$. We thus have that 
\begin{align*}
    & \frac{\partial g(M)}{\partial w_{p}} = \sum_{m, n} M_{m, n} \frac{\partial{[A f_Z(\mathbf{W})]_{v(m, n)}}}{\partial w_{p}} = \sum_{m, n} M_{m, n}  \sum_{\ell=1}^{pq} A_{v(m,n), \ell} \frac{\partial{f_Z(\mathbf{W})_{\ell}}}{\partial w_{p}} \\
    & \implies  K( M_{ij}, M_{i'j'}) = \left\langle \sum_{a=1}^{pq}  A_{v(i,j), a} \frac{\partial f_Z(\mathbf{W})_{a}}{\partial \textbf{w}}, \sum_{b=1}^{pq}  A_{v(i',j'), b} \frac{\partial f_Z(\mathbf{W})_{b}}{\partial \textbf{w}} \right\rangle  = A K_g A^T,
\end{align*}
which completes the proof.
\end{proof}

While the Proposition above generally implies that a a linear transformation requires evaluating a quadratic form when computing the CNTK, the matrix $A$ corresponding to layers used in practice is typically extremely sparse.  Hence, the required computation is simplified drastically, as is demonstrated by the following corollaries (the proofs follow directly from the proposition above).  

% We will make use of the recursive formula for the CNTK in deriving the formulas for these layers.  In particular, we will analyze the impact of each of these layers on the terms $\Sigma, \dot{\Sigma}, K$ used in the recursions in Section \ref{sec: Derivation of the CNTK for Image Inpainting}.  For example, if layer $\ell$ is a downsampling layer, the CNTK at depth $\ell$ will be a sub-sampled version of the kernel at depth $\ell - 1$.  This is similar to the case for computing the covariance of the neural network Gaussian process described in Section 3.3 of \cite{BayesianDeepImagePrior}.  In all results below, we will assume that the tensors $\Sigma, \dot{\Sigma}, K$ are zero-indexed.  

\begin{corollary*}[Downsampling through Strided Convolution] 
\label{corollary: Downsampling CNTK}
Let $\Sigma^{(\ell)}, \dot{\Sigma}^{(\ell)}, K^{(\ell)} \in \mathbb{R}^{d \times d \times d \times d}$ correspond to the tensors used in the CNTK for a depth $\ell$ convolutional network.  Then, using downsampling with a stride of $2$ at step $\ell+1$ maps the tensors to $\Sigma^{(\ell + 1)}, \dot{\Sigma}^{(\ell + 1)}, K^{\ell +1} \in \mathbb{R}^{\frac{d}{2} \times \frac{d}{2} \times \frac{d}{2} \times \frac{d}{2}}$ as follows: $\forall ~ i, j, i', j' \equiv 0 ~ (\hspace{-2mm}\mod 2)$,
\begin{align*}
    \Sigma^{(\ell + 1)}\left(\frac{i}{2}, \frac{j}{2}, \frac{i'}{2}, \frac{j'}{2}\right) &=  \Sigma^{(\ell)}(i, j, i', j'),\\
    \dot{\Sigma}^{(\ell + 1)}\left(\frac{i}{2}, \frac{j}{2}, \frac{i'}{2}, \frac{j'}{2}\right) &=  \dot{\Sigma}^{(\ell)}(i, j, i', j'), \\
    K^{(\ell + 1)}\left(\frac{i}{2}, \frac{j}{2}, \frac{i'}{2}, \frac{j'}{2}\right) &=  K^{(\ell)}(i, j, i', j').  
\end{align*}
\end{corollary*}

\begin{corollary*}[Nearest Neighbor Upsampling] 
\label{corollary: Upsampling CNTK}
Let $\Sigma^{(\ell)}, \dot{\Sigma}^{(\ell)}, K^{(\ell)} \in \mathbb{R}^{\frac{d}{2} \times \frac{d}{2} \times \frac{d}{2} \times \frac{d}{2}}$ correspond to the tensors used in the CNTK for a depth $\ell$ convolutional network.  Then, using nearest neighbor upsampling with a scale factor of $2$ at step $\ell+1$ transforms the tensors to $\Sigma^{(\ell+1)}, \dot{\Sigma}^{(\ell+1)}, K^{(\ell +1)} \in \mathbb{R}^{d \times d \times d \times d}$ as follows: 
\begin{align*}
    \Sigma^{(\ell+1)}\left(i, j, i', j' \right) &=  \Sigma^{(\ell)}\left(\floor*{\frac{i}{2}}, \floor*{\frac{j}{2}}, \floor*{\frac{i'}{2}}, \floor*{\frac{j'}{2}} \right),\\
    \dot{\Sigma}^{(\ell + 1)}\left(i, j, i', j' \right) &=  \dot{\Sigma}^{(\ell)}\left(\floor*{\frac{i}{2}}, \floor*{\frac{j}{2}}, \floor*{\frac{i'}{2}}, \floor*{\frac{j'}{2}} \right),\\
    K^{(\ell + 1)}\left(i, j, i', j' \right) &=  K^{(\ell)}\left(\floor*{\frac{i}{2}}, \floor*{\frac{j}{2}}, \floor*{\frac{i'}{2}}, \floor*{\frac{j'}{2}} \right).  
\end{align*}
\end{corollary*}

The computation for bilinear upsampling (Ch. 2.4 of \cite{DigitalImageProcessing}) is presented below. We primarily use the structure of the updates to $\Sigma, \dot{\Sigma}, K$ to efficiently compute the CNTK when the channels of $X$ are drawn i.i.d. from a stationary distribution.

When bilinearly upsampling (Ch. 2.4 of \cite{DigitalImageProcessing}) an image $A\in \mathbb{R}^{d \times d}$ to an image $\tilde{A} \in \mathbb{R}^{2d \times 2d}$, each coordinate of $\tilde{A}$ is a linear combination of four coordinates of $A$.  Namely for $\alpha = \frac{d-1}{2d-1}$, 
\begin{align*}
    \tilde{A}_{i,j} = \sum_{a, b \in \{0, 1\}} \lambda^{(i,j)}_{a, b} A_{\lfloor \alpha i \rfloor  + a, \lfloor \alpha j \rfloor + b},
\end{align*}
and $\lambda_{a,b}^{i,j}$ is selected as follows.  Let  $r = \lfloor \alpha i \rfloor, c = \lfloor \alpha j \rfloor$ and let:
\begin{align*}
    \ell_r &= \frac{r}{\alpha}, u_r = \frac{r+1}{\alpha},  \ell_c = \frac{c}{\alpha}, u_c = \frac{c+1}{\alpha}, \\
    X &= [u_r - r, r - \ell_r], Y = [u_c - c, c - \ell_c], C = \frac{1}{(u_r - \ell_r) (u_c - \ell_c)}.
\end{align*}
Then, $\lambda_{a, b}^{(i,j)} = C X_a Y_b$ for $a, b \in \{0, 1\}$.   The CNTK tensors are now transformed as follows. 

\begin{corollary}[Bilinear Upsampling] Let $\Sigma^{(\ell)}, \dot{\Sigma}^{(\ell)}, K^{(\ell)} \in \mathbb{R}^{\frac{d}{2} \times \frac{d}{2} \times \frac{d}{2} \times \frac{d}{2}}$ correspond to the tensors used in the CNTK for a depth $\ell$ convolutional network.  Then, using bilinear upsampling with a scale factor of $2$ at step $\ell+1$ transforms the tensors to $\Sigma^{(\ell + 1)}, \dot{\Sigma}^{(\ell + 1)}, K^{\ell +1} \in \mathbb{R}^{d \times d \times d \times d}$ as follows: 
\begin{align*}
    \Sigma^{(\ell + 1)}\left(i, j, i', j' \right) &=  \sum_{a, b \in \{0, 1\}} \sum_{a', b' \in \{0, 1\}} \lambda_{a,b}^{(i,j)} \lambda_{a',b'}^{(i', j')} \Sigma^{(\ell)}\left(\lfloor \alpha i \rfloor+ a, \lfloor \alpha j \rfloor + b, \lfloor \alpha i' \rfloor + a',  \lfloor \alpha j' \rfloor + b' \right),\\
    \dot{\Sigma}^{(\ell + 1)}\left(i, j, i', j' \right) &=  \sum_{a, b \in \{0, 1\}} \sum_{a', b' \in \{0, 1\}} \lambda_{a,b}^{(i,j)} \lambda_{a',b'}^{(i', j')} \dot{\Sigma}^{(\ell)}\left(\lfloor \alpha i \rfloor+ a, \lfloor \alpha j \rfloor + b, \lfloor \alpha i' \rfloor + a',  \lfloor \alpha j' \rfloor + b' \right),\\
    K^{(\ell + 1)}\left(i, j, i', j' \right) &=  \sum_{a, b \in \{0, 1\}} \sum_{a', b' \in \{0, 1\}} \lambda_{a,b}^{(i,j)} \lambda_{a',b'}^{(i', j')} K^{(\ell)}\left(\lfloor \alpha i \rfloor+ a, \lfloor \alpha j \rfloor + b, \lfloor \alpha i' \rfloor + a',  \lfloor \alpha j' \rfloor + b' \right).\\
\end{align*}
\end{corollary}

\section{Efficient Computation of the CNTK for High Resolution Images}
\label{appendix: M}

Computing and storing the CNTK exactly for high resolution images is computationally prohibitive when using a naive approach.  In particular, \cite{BayesianDeepImagePrior} notes that the kernel $K$ for a $500 \times 500$ ($K \in \mathbb{R}^{500 \times 500 \times 500 \times 500}$) resolution image requires roughly 233GB of memory, which is infeasible on common hardware.  In order to overcome these computational limitations, \cite{DenoisingNTK} uses the Nyström method \cite{Nystrom} to approximate the kernel.  In this section, we will demonstrate that we can compute the exact CNTK in a memory and run-time efficient manner for any convolutional neural network with circular padding, strided convolution, and nearest neighbor upsampling layers by using a \prior $Z$ that has infinitely many channels.

Our key insight is that once architecture is fixed, the the CNTK for low resolution images can be expanded to that for high resolution images.  In particular, when the convolutional architecture can be applied to both images of resolution $d_1$ and $d_2$  with $d_2 > d_1$, we can expand the kernel for resolution $d_1$,  $K_{d_1} \in \mathbb{R}^{d_1 \times d_1 \times d_1 \times d_1}$, to a tensor of size $\mathbb{R}^{d_1 \times d_1 \times d_2 \times d_2}$, which can be indexed to match the entries of the kernel for resolution $d_2$, $K_{d_2} \in  \mathbb{R}^{d_2 \times d_2 \times d_2 \times d_2}$.

% We give a concrete example illustrating the computational benefits of our expansion technique below. 

% \begin{example}
% Let $g_X: \mathbb{R}^{p \times 512 \times 512} \to \mathbb{R}^{3 \times 512 \times 512}$ represent a convolutional neural network used for image inpainting with circular padding, 3 layers of strided convolution with a stride size of $2$ in each direction, and 3 nearest neighbor upsampling layers, that operates on an infinite channel image $X$ such that \begin{align*}
%     \sum_{p=1}^{\infty} X[p, i, j] X[p, i', j'] = \begin{cases}
%     1/300 & i = i' ~,~ j = j' \\
%     1/400 & \text{otherwise}
% \end{cases}
% \end{align*}
% Then, by computing the CNTK for $16 \times 16$ resolution images, $K_{\ell} \in \mathbb{R}^{16 \times 16 \times 16 \times 16}$, we can expand up to the exact CNTK for $512 \times 512$ images by indexing a matrix $\tilde{K} \in \mathbb{R}^{8 \times 8 \times 512 \times 512}$.  Computing $K_{\ell}$ takes roughly $11$ seconds when using a CPU with 1 thread and $\tilde{K}$ uses less than $100$MB of memory with floating point precision.  On the other hand, even storing the true kernel $K \in \mathbb{R}^{512 \times 512 \times 512 \times 512}$ would require roughly 256GB memory when using floating point precision.  
% \end{example}

In order to expand the kernel for low resolution images to the one for high resolution images, we need only pad and permute the rows and columns of the low resolution matrix.  We define the required operations formally below (using zero indexing for our matrices).

\begin{definition} [Row and Column Rotation]
Let $\Pi_{i,j}: \mathbb{R}^{d \times d} \to \mathbb{R}^{d \times d}$ such that $\Pi_{i,j} (A) = P_{\pi_i} A P_{\pi_j}$  where $P_{\pi_\ell}$ is a permutation matrix with permutation $\pi_{\ell}(i) = (i + \ell)\hspace{-1mm} \mod d $.
\end{definition}

\begin{definition}[Minimum Padding]
Let $M: \mathbb{R}^{d_1 \times d_1} \to \mathbb{R}^{d_2 \times d_2}$ with $d_2 \geq d_1$ such that $M(A) = \tilde{A}$, where 
\begin{align*}
\tilde{A}_{i,j} &= \begin{cases} 
A_{i,j} & i < d_1, j < d_1\\
\min_{a, b \in [d_1]} A_{a, b} & \text{otherwise}
\end{cases}. 
\end{align*}
\end{definition}

\begin{example}
The operator $\Pi_{i,j}$ rotates the rows of $A$ down by $i$ and rotates the columns of $A$ right by $j$ as follows: 
\begin{align*}
    A = \begin{bmatrix} 
    A_{1,1} & A_{1,2} & A_{1,3} \\
    A_{2, 1} & A_{2,2} & A_{2,3} \\
    A_{3, 1} & A_{3, 2}& A_{3,3}
    \end{bmatrix} \implies \Pi_{1, 2}(A) = \begin{bmatrix} 
    A_{3, 2} & A_{3, 3} & A_{3,1} \\
    A_{1, 2} & A_{1, 3} & A_{1,1} \\
    A_{2, 2} & A_{2, 3} & A_{2,1}
    \end{bmatrix}.
\end{align*}
Minimum padding $M: \mathbb{R}^{2 \times 2} \to \mathbb{R}^{4 \times 4}$ expands a matrix as follows:
\begin{align*}
    A = \begin{bmatrix}
    0.1 & 0.2 \\
    0.3 & 0.4 \\
    \end{bmatrix}   
    \implies M(A) = \begin{bmatrix}
    0.1 & 0.2 & 0.1 & 0.1  \\
    0.3 & 0.4 & 0.1 & 0.1 \\
    0.1 & 0.1 & 0.1 & 0.1  \\
    0.1 & 0.1 & 0.1 & 0.1     
    \end{bmatrix}.   
\end{align*}
\end{example}

\vspace{0.2cm}
The theorem below demonstrates how to construct the kernel for a high resolution image by expanding and indexing a low resolution kernel.  We assume that all strided convolutional layers have a stride size of $2$ in each direction and all upsampling layers have a scaling factor of $2$.  For the following theorem, we also write the kernel $K \in \mathbb{R}^{mn \times mn}$ as a 4 dimensional tensor $K \in \mathbb{R}^{m \times n \times m \times n}$, where $K(i, j, i', j'):= K(M_{ij}, M_{i'j'})$. 

\begin{theorem*}[CNTK Expansion]
\label{theorem: Kernel Expansion Trick}
Let $g$ denote a convolutional neural network with circular padding, $s$ downsampling with strided convolution layers and $s$ nearest neighbor upsampling layers used to inpaint images in $\mathbb{R}^{2^{s+1} \times 2^{s+1}}$.  Define the \prior $Z^{(\ell)} = \{Z_p^{(\ell)}\}_{p=1}^{\infty} \subset \mathbb{R}^{2^\ell \times 2^\ell}$ for $\ell \in \mathbb{Z}_+$ such that:
\begin{align}
\label{eq: special structure on prior X}
     \sum_{p=1}^{\infty} Z_{p, i, j}^{(\ell)} Z_{p, i', j'}^{(\ell)}  = \begin{cases} 
    C_1 & i = i' ~,~  j = j' \\
    C_2 & \text{otherwise}
    \end{cases}.
\end{align} 
Let $d_2 = 2^{p_2}$ such that $p_2 > s+1$.   For $\alpha = 2^{\beta}$, let $K_{\alpha}$ denote the CNTK for $g$ when used to inpaint images in $\mathbb{R}^{\alpha \times \alpha}$ with \prior $Z^{(\beta)}$. Let $p = 2^s$, $i' = i \hspace{-1mm} \mod p, j' = j \hspace{-1mm} \mod p$.  Then for $i, j \in [d_2]$, we compute $\tilde{K} \in \mathbb{R}^{p \times p \times d_2 \times d_2}$ as follows: 
\begin{align*}
\tilde{K}(i', j', :, :)  =  \Pi_{i' - p, j' - p} ( M ( \Pi_{p - i', p - j'} (K_{s+1}[i', j', :, : ])) ),
\end{align*}
and we have: 
\begin{align*}
    K_{d_2}(i, j, :, :) = \Pi_{i - i', j - j'} \tilde{K}(i', j', :, :).
\end{align*}
\end{theorem*}

\begin{proof}
To provide intuition for the general case, we first prove the result for $s=0$.  Using the Proposition from Appendix \ref{appendix: J} and the conditions on $Z^{(\ell)}$, we obtain 
\begin{align*}
    \Sigma^{(0)}(i,j, i', j') = K^{(0)}(i, j, i', j') = \begin{cases}
    q^2 C_1 ~~ \text{if $i = i', j' = j'$} \\
    q^2 C_2 ~~ \text{otherwise} 
    \end{cases}.
\end{align*}
Hence for any $\ell, \ell' \geq 1$ with $\ell' < \ell$, we conclude that $K_{\ell}(i, j, i', j') = K_{\ell'}(a, b, a', b')$ when $(i, j) \neq (i', j')$ and $(a, b) \neq (a', b')$,  and $K_{\ell}(i, j, i, j) = K_{\ell'}(a, b, a, b)$ for all $i, j  \in [2^\ell]$ and $a, b \in [2^{\ell'}]$.  Hence, by permuting rows, columns and minimum padding $K_{\ell'}$, we can recover the kernel for $K_{\ell}$.  Note that for $\ell = 0$, we do not ever record a kernel entry for the case where $(i, j) \neq (i', j')$ and so minimum padding would pad with the incorrect minimum value of $K_0(0, 0, 0, 0)$.  This is why we need to expand up from the kernel for images of dimension $2^{s+1}$ and not just from the kernel for images of dimension $2^{s}$.  

For $s > 0$, we rely on the nearest neighbor upsampling and downsampling corollaries from Appendix \ref{appendix: L} to understand which entries of $K_{\ell}(i, j, i', j')$ are equal to $K_{\ell'}(a, b, a', b')$.  Since $Z^{(\ell)}, Z^{(\ell')}$ have the same range $\{C_1, C_2\}$ of channel-wise products, it suffices to identify the elements of $K_{\ell'}$ that are equal.  These elements will then naturally be equal in $K_{\ell}$ after minimum padding.  

From \cite{BayesianDeepImagePrior}, we have that down-sampling through strided convolution preserves stationarity, and so after $t$ downsampling and convolutional layers, we again have that $K_{\ell}^{(t)}(i, j, i', j') = K_{\ell'}^{(t)}(a, b, a', b')$ when $(i, j) \neq (i', j')$ and $(a, b) \neq (a', b')$,  and $K_{\ell}^{(t)}(i, j, i, j) = K_{\ell'}^{(t)}(a, b, a, b)$ for all $i, j  \in [2^\ell]$ and $a, b \in [2^{\ell'}]$.  

In general, upsampling (including nearest neighbor upsampling) does not preserve stationarity, as is discussed in \cite{BayesianDeepImagePrior}.  However,  nearest neighbor upsampling preserves equality (up to permutation) between $K_{\ell}(i, j, :, :)$ and $K_{\ell}(i', j', :, :)$ provided that $i \equiv i' ~ (\hspace{-2mm}\mod 2^s)$ and $j \equiv j' ~ (\hspace{-2mm}\mod 2^s)$.  This follows immediately from analyzing the output after nearest neighbor upsampling in the original image space.  In the following, we provide an example. 
\begin{example} 
Consider the output of nearest neighbor upsampling a single channel $Y \in \mathbb{R}^{2 \times 2}$ to $\tilde{Y} \in \mathbb{R}^{4 \times 4}$ :
\begin{align*}
    Y = \begin{bmatrix} 
    Y_{0, 0} & Y_{0, 1} \\
    Y_{1, 0} & Y_{1, 1}    
    \end{bmatrix} \implies \tilde{Y} = \begin{bmatrix} 
    Y_{0, 0} & Y_{0, 0} & Y_{0, 1} & Y_{0, 1} \\
    Y_{0, 0} & Y_{0, 0} & Y_{0, 1} & Y_{0, 1} \\
    Y_{1, 0} & Y_{1, 0} & Y_{1, 1} & Y_{1, 1} \\
    Y_{1, 0} & Y_{1, 0} & Y_{1, 1} & Y_{1, 1}
    \end{bmatrix}.
\end{align*}
From the stationarity of $Z^{(\ell)}$ and since convolution and downsampling layers preserve stationarity, we have that the CNTK for the above output $K_2(i, j, :, :)$  equals (up to permutation) $K_2(i', j', :, :)$ whenever $i \equiv i' ~ (\hspace{-2mm}\mod 2)$ and $j \equiv j' ~ (\hspace{-2mm}\mod 2)$ since the corresponding entries in $\tilde{Y}$ have identical patterns of neighbors (i.e. a row or column permutation by $2^s$ does not affect the sums involved in the kernel computation).  
\end{example}

Thus, we conclude that the range of entries in $K_{\ell'}(a, b, :, :)$ and $K_{\ell}(i, j, :, :)$ are equal whenever both $i \equiv a ~ (\hspace{-2mm}\mod 2^s)$ and $b \equiv j ~ (\hspace{-2mm}\mod 2^s)$.  To complete the proof, we just permute and minimum pads the entries of $K_{\ell'}(a, b, :, :)$ to align the expanded matrix such that entry $K_{\ell'}(a, b, a, b)$ corresponds to $K_{\ell}(i, j, i, j)$ in the expanded matrix.  
\end{proof}

\textbf{Remarks.} Note that the expansion trick provided in the theorem above solely depends on (1) the number of downsampling and nearest neighbor upsampling layers; (2) the \prior $Z$ having special structure as described in \eqref{eq: special structure on prior X}; and (3) the convolutional layers using circular padding.  It importantly does \emph{not} depend on the number of layers, type of homogeneous activation function (i.e. ReLU or leakyReLU), or size of the convolutional filters used.  Hence our expansion technique can be used on a range of architectures, as we also demonstrate in Section 4 of the main text.  The permutations $\Pi_{p - i', p - j'},  \Pi_{i' - p, j' - p}$ used to compute $\tilde{K}$ are essentially used to ensure that we perform minimum padding appropriately for kernel values at the kernel's edges.  Lastly, when there are $s$ downsampling and upsampling layers, the smallest image size we can expand from is an image of size $2^{s+1} \times 2^{s+1}$.  We cannot use images of size $2^{s}$ since the corresponding kernel will not contain the same minimum value as that for images of size $2^{s+1}$.

\section{Experimental Details for Image Inpainting}
\label{appendix: N}
In the following, we  describe the hyperparameters used for training neural networks and solving kernel regression with the CNTK on the considered image inpainting and image reconstruction tasks.  

\subsection{Large Hole Inpainting}  

For all large hole inpainting experiments, we used the autoencoder architecture from \cite{BayesianDeepImagePrior} that has 6 downsampling and upsampling layers with no skip connections.  On all images other than the ``library'' image, we trained using the Adam optimizer \cite{Adam} for 1000 epochs with a learning rate $10^{-2}$.  For the ``library'' image, we trained using the Adam optimizer for 6000 epochs with a learning rate of $10^{-2}$.  We used a random seed of 15 for all libraries.  For implementing Adam with Langevin dynamics, we used the code and data from \cite{BayesianDeepImagePrior} directly.  We performed optimal early stopping for all neural networks, i.e. we chose the reconstruction that has the closest match in PSNR to the ground truth.  While impossible to perform in practice, optimal early stopping allows us to compare the CNTK with the best possible result from the neural network. 

For solving kernel regression with the CNTK, we trained using EigenPro \cite{EigenPro, EigenProGPU} for 10 epochs, i.e., we did not early stop for large hole inpainting tasks. We scaled all kernels by a factor of $0.5$ to ensure convergence with EigenPro.    

\subsection{Image Reconstruction} 

Below we list the architectures and training procedure for each image.  For the neural networks, we always trained for 6000 epochs using Adam with a learning rate of $10^{-3}$, which is the learning rate used in \cite{DeepImagePrior}.  All neural networks have $128$ convolutional filters per layer as is the case in \cite{DeepImagePrior}.  We trained the CNTK for the corresponding architecture with EigenPro for 50 epochs, unless otherwise specified.  The architectures used nearest neighbor upsampling, unless otherwise specified. We observed that training longer or, ideally, direct solving kernel regression with the CNTK for networks with nearest neighbor upsampling led to the best PSNR results for image reconstruction tasks.  This is consistent with \cite{DeepImagePrior} in which networks for image reconstruction are trained twice as long as those for large hole inpainting.  A direct solve was only computationally feasible on $256 \times 256$ resolution images.

\begin{itemize}
    \item ``Barbara'': We use a network with 2 downsampling and upsampling layers.  
    \item ``Boat'': We use a network with 6 downsampling and upsampling layers.  We train the CNTK for 100 epochs.
    \item ``Camera Man'': We use a network with 6 downsampling and upsampling layers.  We train the CNTK for 100 epochs.
    \item ``Couple'': We use a network with 6 downsampling and upsampling layers.  We train the CNTK for 100 epochs.
    \item ``Finger'': We use a network with 3 downsampling and upsampling layers.  We train the CNTK for 100 epochs.
    \item ``Hill'': We use a network with 6 downsampling and upsampling layers. 
    \item ``House'': We use a network with 6 downsampling and upsampling layers.  We solve kernel regression exactly using the numpy solve function.  
    \item ``Lena'': We use a network with 6 downsampling and upsampling layers.
    \item ``Man'': We use a network with 6 downsampling and upsampling layers.
    \item ``Montage'': We use a network with 6 downsampling and upsampling layers.  We solve kernel regression exactly using the numpy solve method.  
    \item ``Peppers'': We use a network with 5 downsampling and upsampling layers with bilinear upsampling.  We solve kernel regression exactly using the numpy solve method, but add diagonal regularization from \cite{FiniteVsInfiniteNeuralNetworks}.  In particular, for kernel $K \in \mathbb{R}^{p \times p}$, we add $\frac{4 \cdot 10^{-5}}{p}  tr(K) I_{p \times p}$ to the kernel before using the numpy solve function.
\end{itemize}

\begin{figure}
    \centering
    \includegraphics[height=.8\textheight]{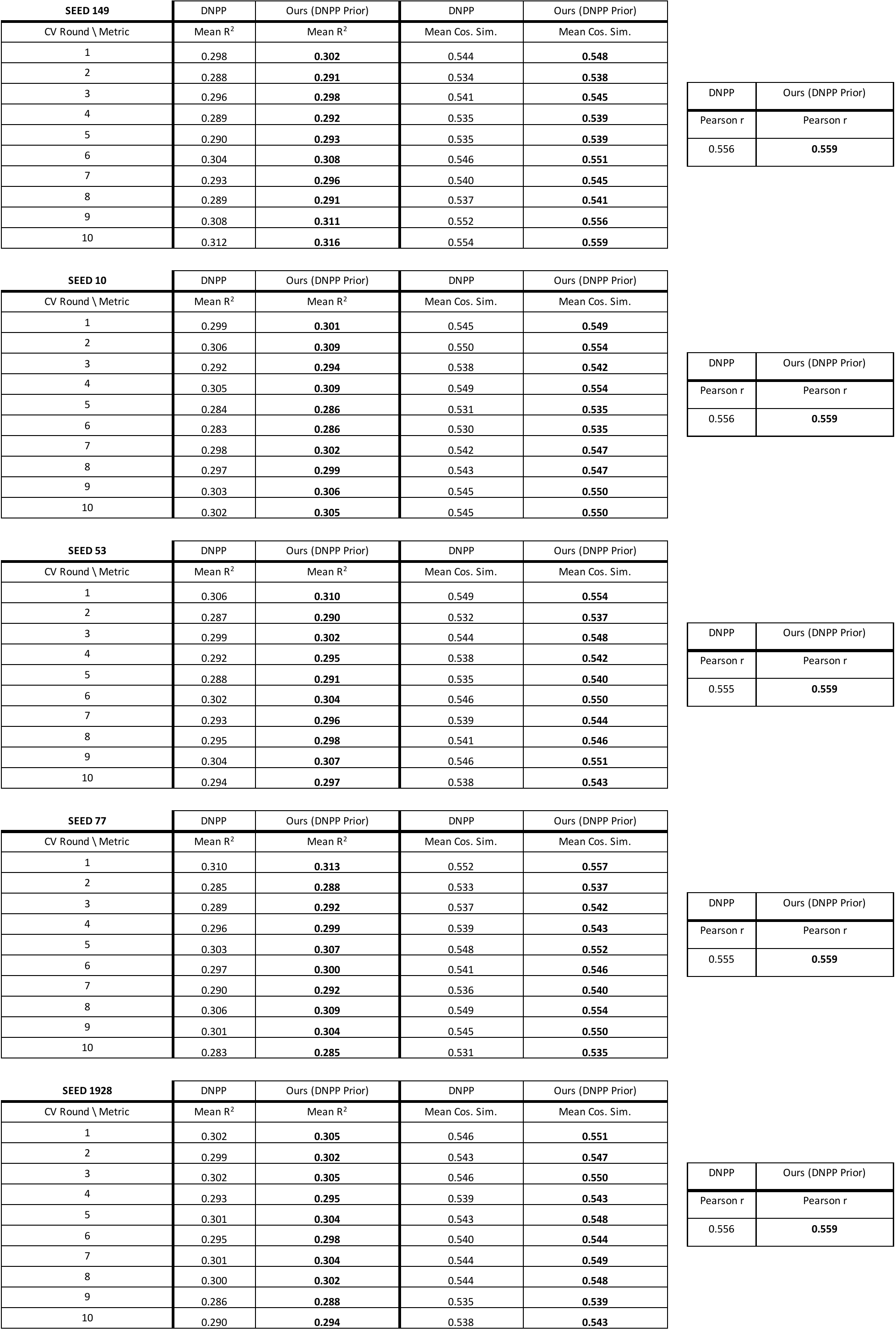}
    \caption{Comparison between DNPP and using the output of DNPP as a \prior.  Using our method with the output of DNPP as a \prior leads to an improvement in all metrics across every round of 10-fold cross validation in 5 seeds.}
    \label{fig: SI DNPP vs. DNPP Prior}
\end{figure}

\begin{figure}
    \centering
    \includegraphics[height=.8\textheight]{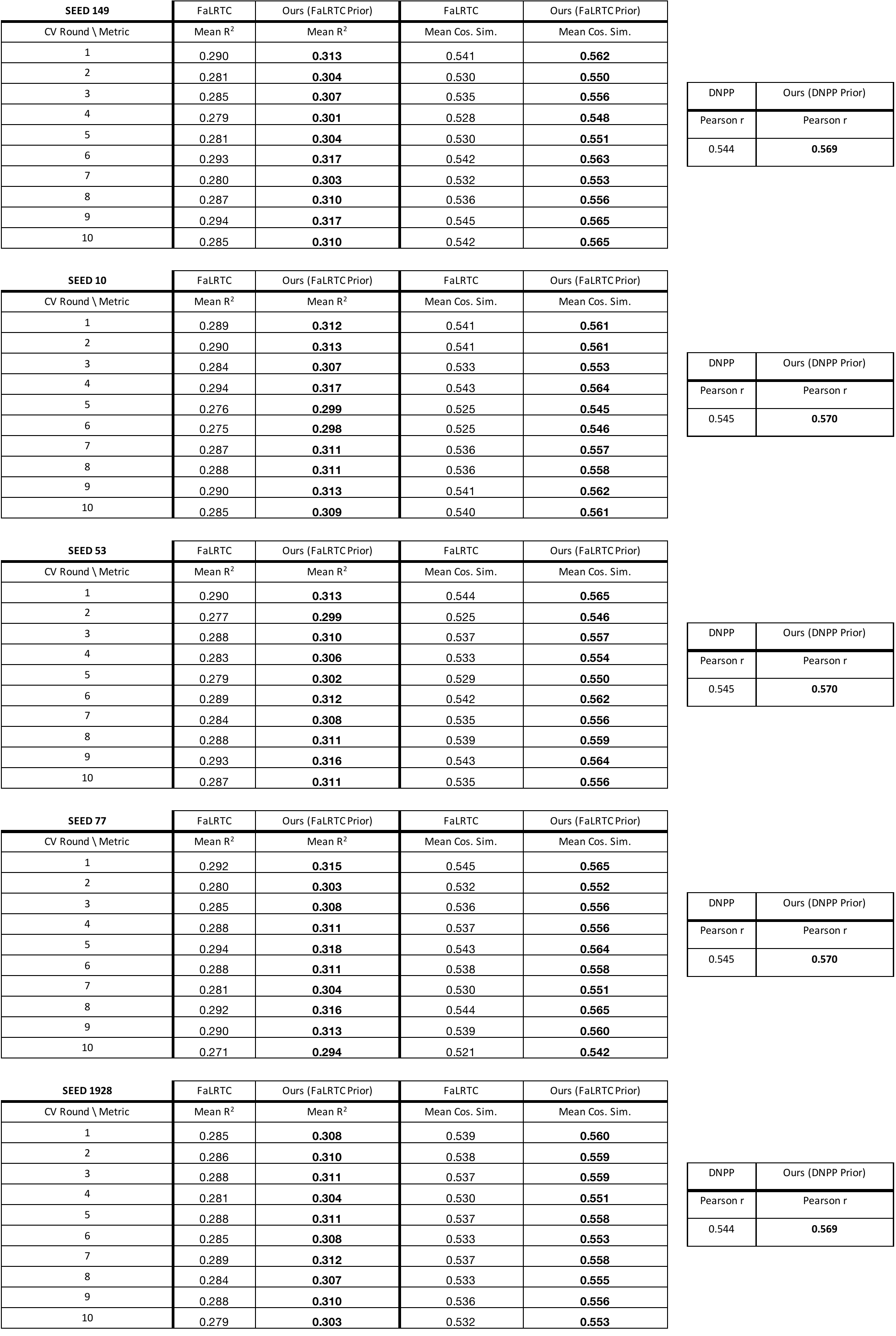}
    \caption{Comparison between FaLRTC and using the output of FaLRTC as a \prior. Using our method with the output of FaLRTC as a \prior leads to an improvement in all metrics across every round of 10-fold cross validation in 5 seeds.}
    \label{fig: SI FaLRTC vs. FaLRTC Prior}
\end{figure}

\begin{figure}
    \centering
    \includegraphics[width=.9\textwidth]{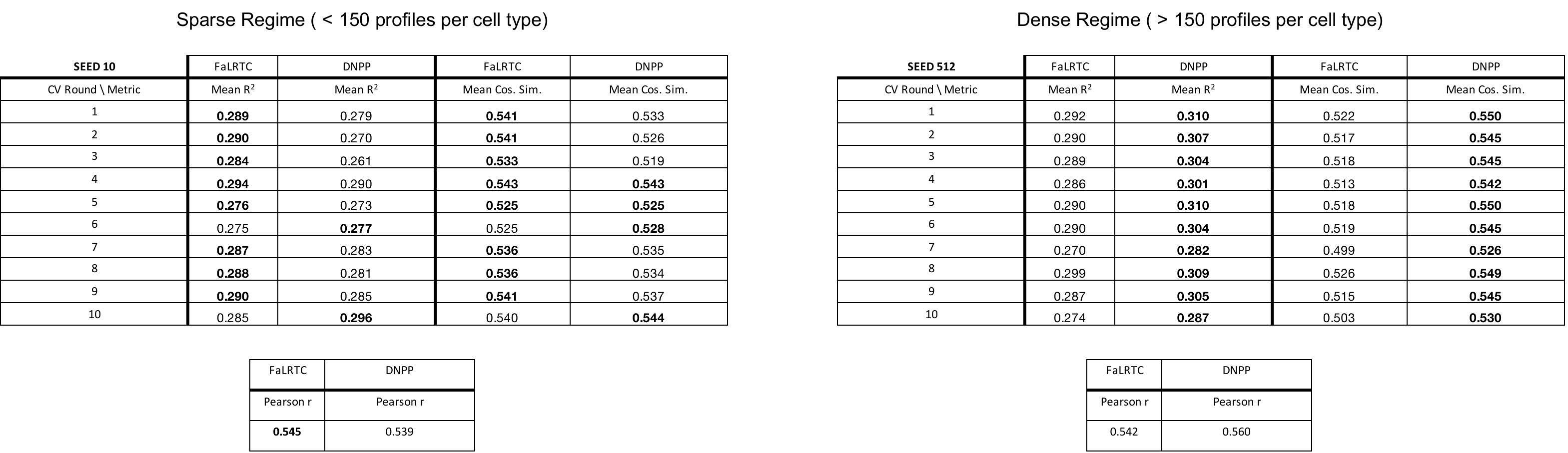}
    \caption{Comparison between DNPP and FaLRTC in (a) the sparse regime (< 150 profiles per cell type) and (b) the dense regime (> 150 profiles per cell type).  We observe that FaLRTC outperforms DNPP in almost every fold for all performance metrics in the sparse regime.  On the other hand DNPP outperforms FaLRTC in the dense regime in every fold for all performance metrics.  This result demonstrates that DNPP can be improved drastically in the sparse regime.}
    \label{fig: SI Dense vs. Sparse Regime}
\end{figure}

\begin{figure}
    \centering
    \includegraphics[width=\textwidth]{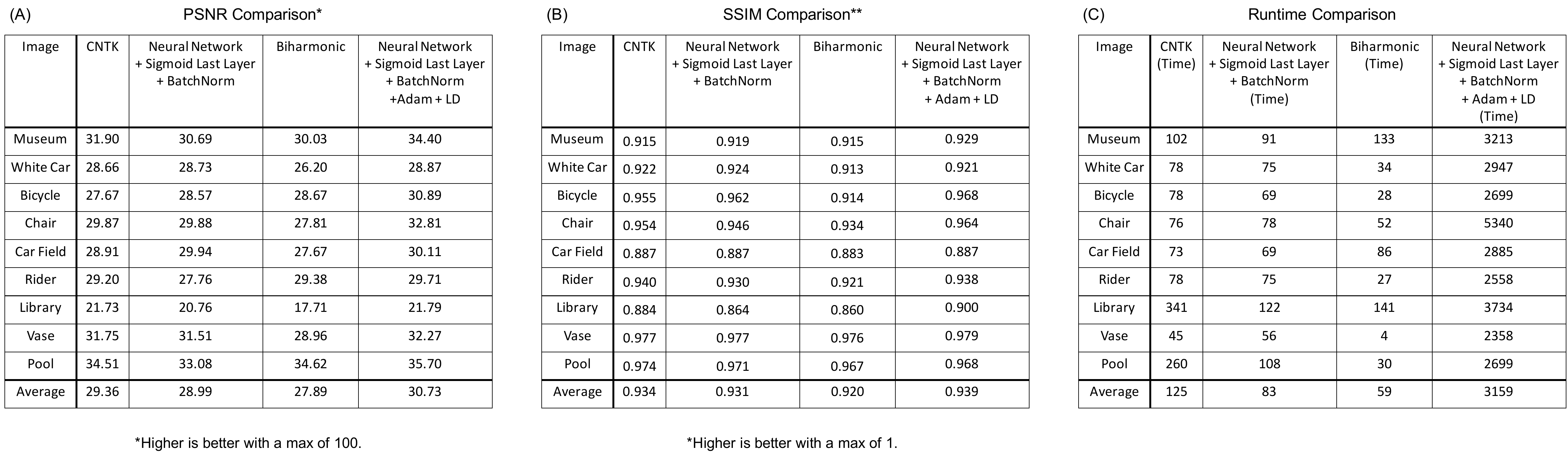}
    \caption{A comparison of PSNR, SSIM, and runtime for large hole image inpainting using our framework (CNTK), corresponding finite width neural networks, and biharmonic inpainting.  We observe that the CNTK outperforms (in PSNR and SSIM) on average both biharmonic inpainting and finite neural networks with sigmoid last layer and batch normalization layers while maintaing a runtime that is comparable to these methods.  The last column illustrates that using more advanced techniques such as Adam with Langevin dynamics \cite{BayesianDeepImagePrior} can be used to boost the performance of neural networks, but at additional computational cost.}
    \label{fig: SI Large Hole Inpainting Tables}
\end{figure}

\begin{figure}
    \centering
    \includegraphics[width=\textwidth]{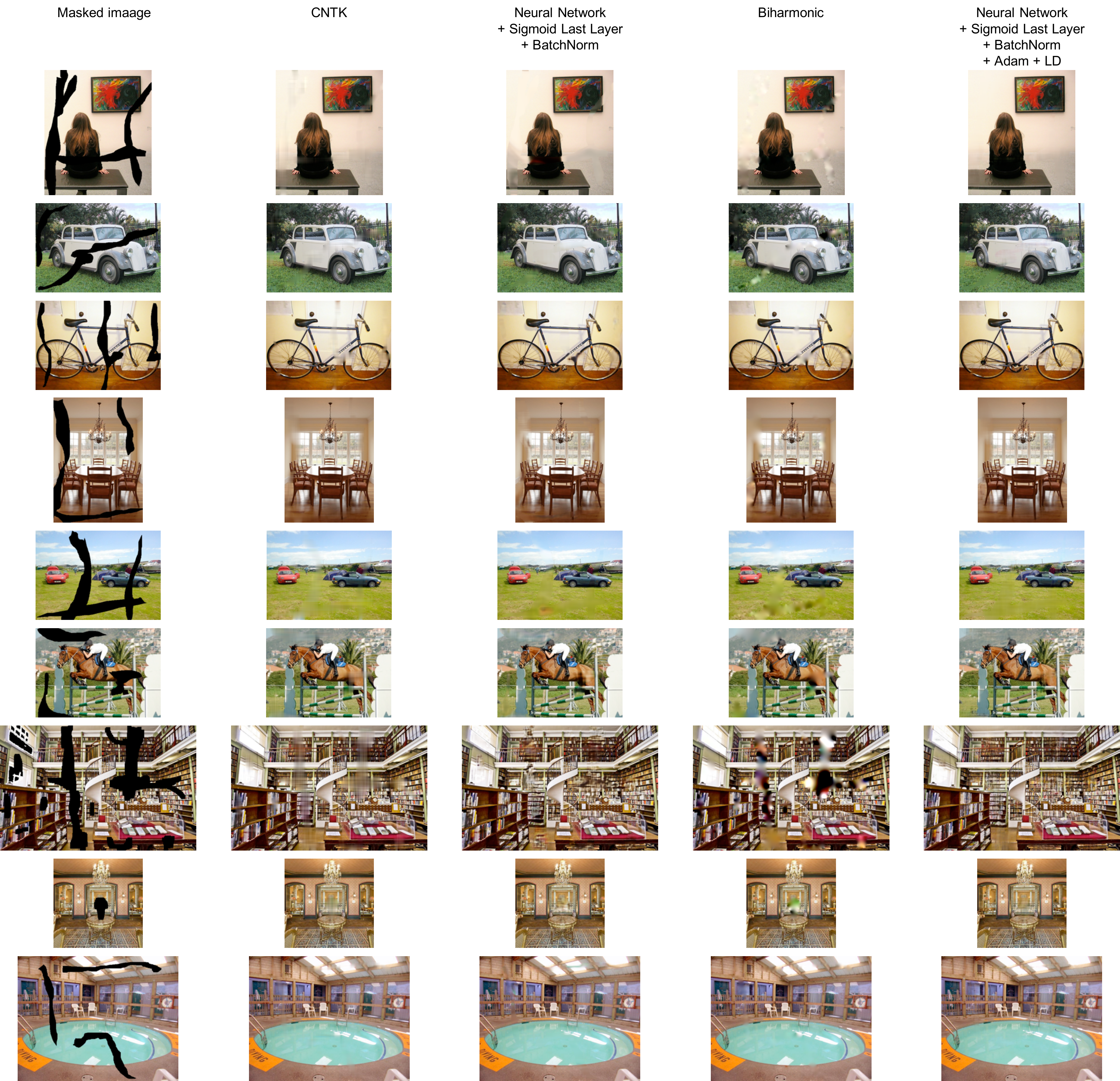}
    \caption{A qualitative comparison of large hole image inpainting using our framework (CNTK), corresponding finite width neural networks, and biharmonic inpainting. See Fig.~\ref{fig: SI Large Hole Inpainting Tables} for the corresponding quantitative comparison.}
    \label{fig: SI Large Hole Inpainting Images}
\end{figure}

\begin{figure}
    \centering
    \includegraphics[width=\textwidth]{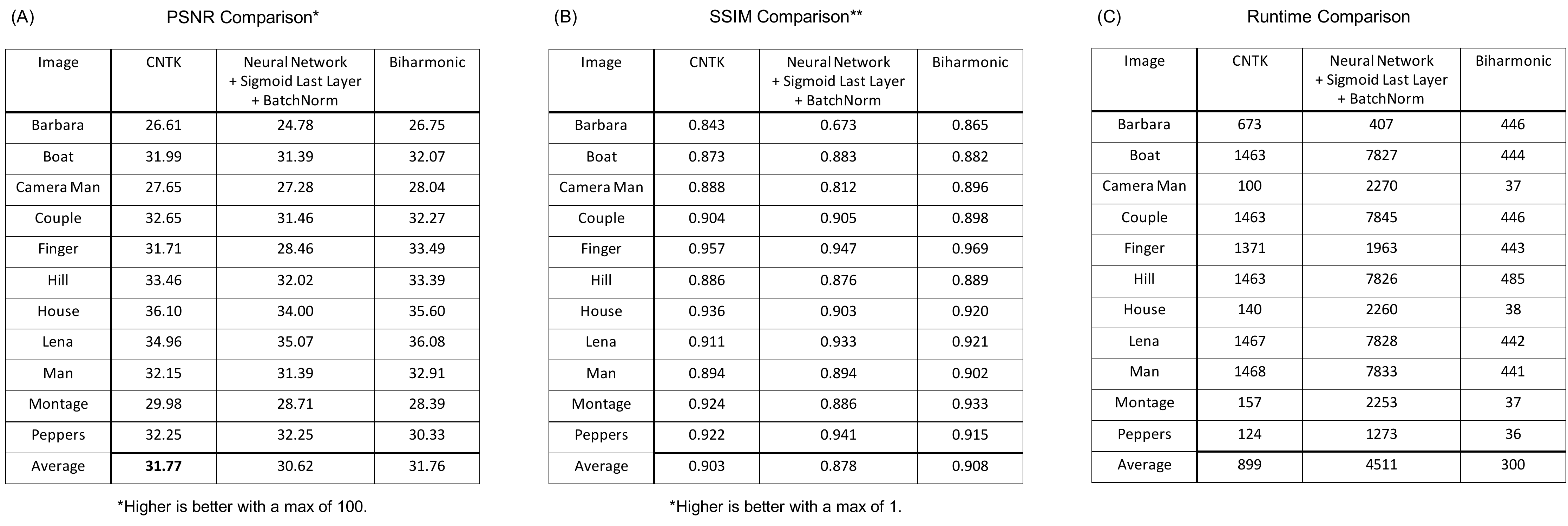}
    \caption{A comparison of PSNR, SSIM, and runtime for image reconstruction using our framework (CNTK), corresponding finite width neural networks, and biharmonic inpainting.  We observe that the CNTK performs (in PSNR and SSIM) on average comparably to biharmonic inpainting and outperforms corresponding finite neural networks with sigmoid last layer and batch normalization layers.  While our method is slower than biharmonic inpainting, it is more flexible than this method (see Fig.~\ref{fig: SI Large Hole Inpainting Tables}), and it is in average much faster than training finite width neural networks for this application.}
    \label{fig: SI Image Reconstruction Table}
\end{figure}

\begin{figure}
    \centering
    \includegraphics[width=.7\textwidth]{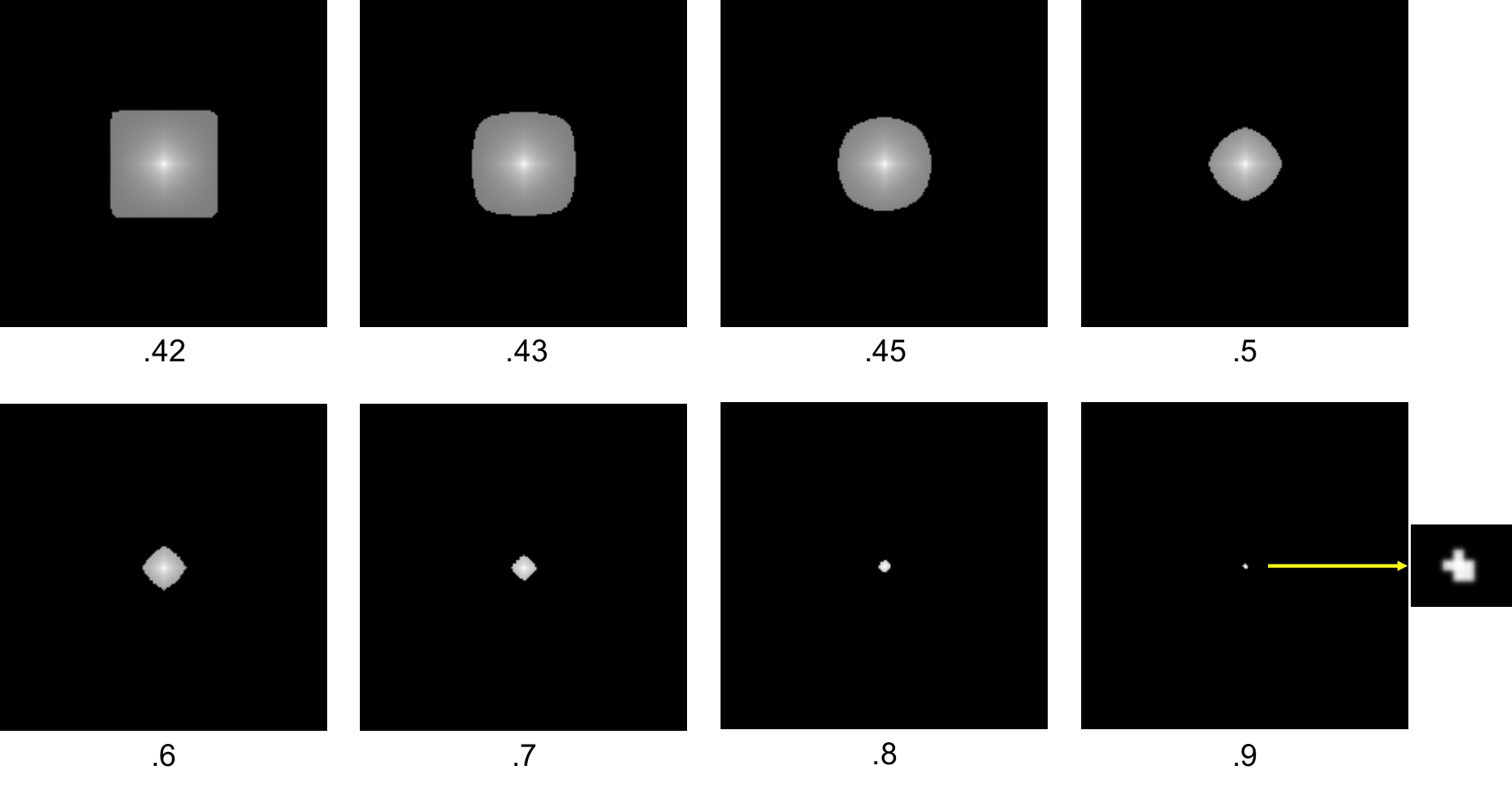}
    \caption{Visualizing the CNTK of a neural network with nearest neighbor downsampling and upsampling layers and a uniform random \prior illustrates that this kernel is akin to a kernel that uses different norms for image completion.  In the above figure, we visualize the CNTK heatmap for coordinate $(64, 64)$ of the CNTK for a neural network with 5 nearest neighbor downsampling and upsampling layers operating on $128 \times 128$ images.  In each subfigure, we zero out the $x$ percentile (provided below each image) of pixel values.  For example, the image on the bottom right corresponds to zeroing out all pixels with values below the $90$th percentile.  We observe that balls of varying norms appear in this visualization: e.g., the $\ell_{\infty}$ ball appears in the upper left and an $\ell_p$ ball with $1 < p < 2$ on the upper right.}
    \label{fig: SI Norms Kernels}`
\end{figure}

\end{document}